\documentclass[twoside,11pt]{article}

\usepackage{jmlr2e}

\usepackage{amsmath}
\usepackage{amssymb}
\usepackage[T1]{fontenc}
\usepackage{csquotes}
\usepackage{algorithmic,algorithm}
\usepackage{mathtools}
\usepackage{booktabs}
\usepackage{subcaption}
\usepackage{multirow}
\usepackage{adjustbox}
\usepackage[]{mdframed}

\newtheorem{assumption}{\textbf{Assumption}}

\begin{document}

\title{Stabilizing Sharpness-aware Minimization Through A Simple Renormalization Strategy}


\author{\name Chengli Tan \email cltan023@outlook.com \\
	\addr School of Mathematics and Statistics\\
	Xi'an Jiaotong University\\
	Xi'an, 710049, China
	\AND
	\name Jiangshe Zhang \email jszhang@mail.xjtu.edu.cn \\
	\addr School of Mathematics and Statistics\\
	Xi'an Jiaotong University\\
	Xi'an, 710049, China
	\AND
	\name Junmin Liu \email junminliu@mail.xjtu.edu.cn \\
	\addr School of Mathematics and Statistics\\
	Xi'an Jiaotong University\\
	Xi'an, 710049, China
	\AND
	\name Yicheng Wang \email ycwang@stu.xjtu.edu.cn \\
	\addr School of Mathematics and Statistics\\
	Xi'an Jiaotong University\\
	Xi'an, 710049, China
	\AND
	\name Yunda Hao \email yunda@cwi.nl \\
	\addr Department of Machine Learning\\
	Centrum Wiskunde \& Informatica\\
	Amsterdam, 1098 XG, the Netherlands
}

\editor{N/A}

\maketitle

\begin{abstract}
	Recently, sharpness-aware minimization (SAM) has attracted much attention because of its surprising effectiveness in improving generalization performance.
	However, compared to stochastic gradient descent (SGD), it is more prone to getting stuck at the saddle points, which as a result may lead to performance degradation.
	To address this issue, we propose a simple renormalization strategy, dubbed Stable SAM (SSAM), so that the gradient norm of the descent step maintains the same as that of the ascent step.
	Our strategy is easy to implement and flexible enough to integrate with SAM and its variants, almost at no computational cost.
	With elementary tools from convex optimization and learning theory, we also conduct a theoretical analysis of sharpness-aware training, revealing that compared to SGD, the effectiveness of SAM is only assured in a limited regime of learning rate.
	In contrast, we show how SSAM extends this regime of learning rate and then it can consistently perform better than SAM with the minor modification.
	Finally, we demonstrate the improved performance of SSAM on several representative data sets and tasks.
\end{abstract}


\section{Introduction}
Over the last decade, deep neural networks have been successfully deployed in a variety of domains, ranging from object detection \citep{redmon2016you}, machine translation \citep{dai2019transformer}, to mathematical reasoning \citep{davies2021advancing}, and protein folding \citep{jumper2021highly}.
Generally, deep neural networks are applied to approximate an underlying function that fits the training set well. In the realm of supervised learning, this is equivalent to solving an unconstrained optimization problem
\begin{equation*}
	\label{eq: finite sum}
	\min_{\boldsymbol{w}} F_S(\boldsymbol{w}) = \frac{1}{n} \sum_{i=1}^n f(\boldsymbol{w}, z_i),
\end{equation*}
where $f$ represents the per-example loss, $\boldsymbol{w}\in\mathbb{R}^d$ denotes the parameters of the deep neural network, and $n$ feature/label pairs $z_i=(x_i, y_i)$ constitute the training set $S$.
Often, we assume each example is i.i.d.~generated from an unknown data distribution $\mathfrak{D}$.
Since deep neural networks are usually composed of many hidden layers and have millions (even billions) of learnable parameters, it is quite a challenging task to search for the optimal values in such a high-dimensional space.
\par
In practice, due to the limited memory and time, we cannot save the gradients of each example in the training set and then apply determined methods such as gradient descent (GD) to train deep neural networks.
Instead, we use only a small subset (so-called mini-batch) of the training examples to estimate the full-batch gradient. Then, we employ stochastic gradient-based methods to make training millions (even billions) of parameters feasible.
However, the generalization ability of the solutions can vary with different training hyperparameters and optimizers. For example, \citet{jastrzkebski2017three,keskar2016large,he2019control} argued that training neural networks with a larger ratio of learning rate to mini-batch size tends to find solutions that generalize better.
Meanwhile, \citet{wilson2017marginal,zhou2020towards} also pointed out that the solutions found by adaptive optimization methods such as Adam \citep{kingma2014adam} and AdaGrad \citep{duchi2011adaptive} often generalize significantly worse than SGD \citep{bottou2018optimization}.
Although the relationship between optimization and generalization remains not fully understood \citep{choi2019empirical,dahl2023benchmarking}, it is generally appreciated that solutions recovered from the flat regions of the loss landscape generalize better than those landing in sharp regions \citep{keskar2016large,chaudhari2019entropy, jastrzebski2021catastrophic, kaddour2022flat}.
This can be justified from the perspective of the minimum description length principle that fewer bits of information are required to describe a flat minimum \citep{hinton1993keeping}, which, as a result, leads to stronger robustness against distribution shift between training data and test data.
\par
Based on this observation, different approaches are proposed towards finding flatter minima, amongst which sharpness-aware minimization (SAM) \citep{foret2020sharpness} substantially improves the generalization and attains state-of-art results on large-scale models such as vision transformers \citep{chen2021vision} and language models \citep{bahri2022sharpness}.
{Unlike standard training that minimizes the loss of the current weight $\boldsymbol{w}_t$, SAM  minimizes the loss of the perturbed weight
	$$\boldsymbol{w}_t^{asc} = \boldsymbol{w}_t + \rho {\nabla F_{\Omega_t}(\boldsymbol{w}_t)},$$ where $\Omega_t$ is a mini-batch of $S$ at $t$-th step and $\rho$ is a predefined constant\footnote{It is worth noting that different from the standard formulation of SAM \citep{foret2020sharpness}, here we drop the normalization term and adopt the unnormalized version \citep{andriushchenko2022towards} for analytical simplicity. While there are some disputes that this simplification sometimes would hurt the algorithmic performance \citep{dai2024crucial, long2024sharpness}, we hypothesize that this is because their analysis is based on GD rather than on SGD. Moreover, the empirical results in Section \ref{sec:Image Classification from Scratch} and from \citet{andriushchenko2022towards} also suggest that the normalization term is not necessary for improving generalization. To avoid ambiguity, we refer to the standard formulation of SAM proposed by \citet{foret2020sharpness} as $\mathrm{SAM}^{\ast}$ where necessary.}.
	\begin{figure}[t]
		\centering
		\begin{subfigure}[b]{0.328\textwidth}
			\centering
			\includegraphics[width=\textwidth, clip, trim= 0 0 0 0]{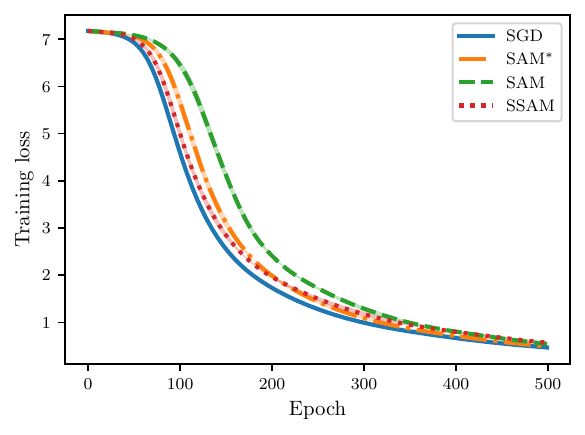}
			\caption{$\rho=0.02$}
		\end{subfigure}
		\begin{subfigure}[b]{0.328\textwidth}
			\centering
			\includegraphics[width=\textwidth, clip, trim= 0 0 0 0]{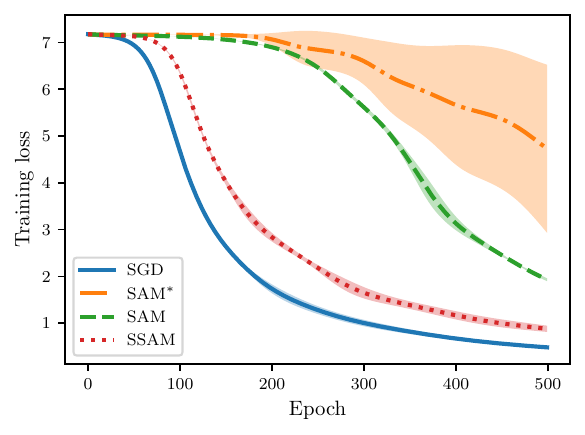}
			\caption{$\rho=0.05$}
		\end{subfigure}
		\begin{subfigure}[b]{0.328\textwidth}
			\centering
			\includegraphics[width=\textwidth, clip, trim= 0 0 0 0]{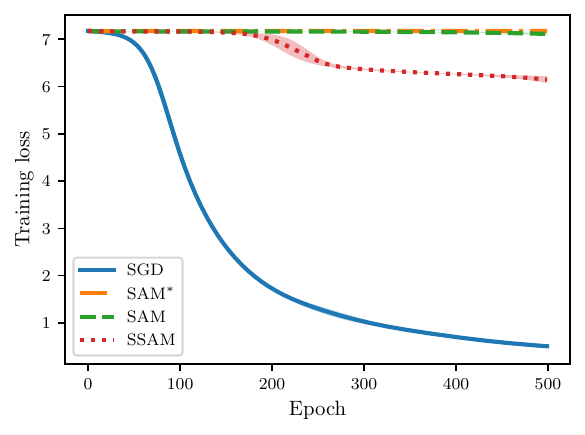}
			\caption{$\rho=0.08$}
		\end{subfigure}
		\caption{Loss curves of different optimizers to escape from the saddle point (namely, the origin) under different values of $\rho$.
			Following \citet{compagnoni2023sde}, we approximate the identity matrix of dimension $d=20$ as the product of two square matrices and initialize them with elements sampled from $\mathcal{N}(0, 1.0e^{-4})$. 
			We then train the linear autoencoder with different optimizers up to 500 epochs using a constant learning rate of $1.0e^{-3}$.
		}
		\label{fig: escaping saddle points}
	\end{figure}
	Despite the potential benefit of improved generalization, however, this unusual operation also brings about one critical issue during training.
	Compared to SGD, as pointed out by \citet{compagnoni2023sde} and \citet{kim2023stability}, SAM dynamics are easier to become trapped in the saddle points and require much more time to escape from them.
	To see this, let us take the linear autoencoder described in \citet{kunin2019loss} as an example.
	It is known that there is a saddle point of the loss function near the origin and here we compare the escaping efficiency of different optimizers.
	As shown in Figure \ref{fig: escaping saddle points}, we can observe that both SAM and $\mathrm{SAM}^\ast$ indeed require more time than SGD to escape from this point and become slower and slower as we gradually increase $\rho$ up to not being able to escape anymore.
}
\par
To stabilize training neural networks with SAM and its variants, here we propose a simple yet effective strategy by rescaling the gradient norm at point $\boldsymbol{w}_t^{asc}$ to the same magnitude as the gradient norm at point $\boldsymbol{w}_t$.
In brief, our contributions can be summarized as follows:
\begin{enumerate}
	\item We proposed a strategy, dubbed Stable SAM (SSAM), to stabilize training deep neural networks with SAM optimizer. 
	Our strategy is easy to implement and flexible enough to be integrated with any other SAM variants, almost at no computational cost.
	Most importantly, our strategy does not introduce any additional hyperparameter, tuning which is quite time-consuming in the context of sharpness-based optimization.
	\item We theoretically analyzed the benefits of SAM over SGD in terms of algorithmic stability \citep{hardt2016train} and found that the superiority of SAM is only assured in a limited regime of learning rate.
	We further extended the study to SSAM and showed that it allows for a larger learning rate and can consistently perform better than SAM under a mild condition. 
	\item We empirically validated the capability of  SSAM to stabilize sharpness-aware training and demonstrated its improved generalization performance in real-world problems. 
\end{enumerate}
The remainder of the study is organized as follows.
Section~\ref{sec:related-works} reviews the related literature, while Section~\ref{sec:methodology} elaborates on the details of the renormalization strategy.
Section~\ref{sec:theoretical analysis} then provides a theoretical analysis of SAM and SSAM from the perspective of expected excess risk.
Finally, before concluding the study, Section~\ref{sec:experiments} presents the experimental results.
\section{Related Works}
\label{sec:related-works}
Building upon the seminal work of SAM \citep{foret2020sharpness}, numerous algorithms have been proposed, most of which can be classified into two categories.
\par
The first category continues to improve the generalization performance of SAM. 
By stretching/shrinking the neighborhood ball according to the magnitude of parameters,  ASAM \citep{kwon2021asam} strengthens the connection between sharpness and generalization, which might break up due to model reparameterization.
Similarly, instead of defining the neighborhood ball in the Euclidean space, FisherSAM \citep{kim2022fisher} runs the SAM update on the statistical manifold induced by the Fisher information matrix.
Since one-step gradient ascent may not suffice to accurately approximate the solution of the inner maximization, RSAM \citep{liu2022random} was put forward by smoothing the loss landscape with Gaussian filters. 
This approach is similar to \citet{haruki2019gradient,bisla2022low}, both of which aim to flatten the loss landscape by convoluting the loss function with stochastic noise.
To separate the goal of minimizing the training loss and sharpness, GSAM \citep{zhuang2021surrogate} was developed to seek a region with both small loss and low sharpness.
Contrary to imposing a common weight perturbation within each mini-batch, $\delta$-SAM \citep{zhou2022sharpness} uses an approximate per-example perturbation with a theoretically principled weighting factor.
\par
The second category is devoted to reducing the computational cost because SAM involves two gradient backpropagations at each iteration.
An early attempt is LookSAM \citep{liu2022towards}, which runs a SAM update every few iterations.
Another strategy is RST \citep{zhao2022randomized}, according to which SAM and standard training are randomly switched with a scheduled probability.
Inspired by the local quadratic structure of the loss landscape, SALA \citep{tan2024sharpness} uses SAM only at the terminal phase of training when the distance between two consecutive steps is smaller than a threshold.
Similarly, AESAM\citep{jiang2022adaptive} designs an adaptive policy to apply SAM update only in the sharp regions of the loss landscape. 
ESAM \citep{du2021efficient} and Sparse SAM \citep{mi2022make} both attempt to perturb a subset of parameters to estimate the sharpness measure, while KSAM \citep{ni2022k} applies the SAM update to the examples with the highest loss.
Another intriguing approach is SAF \citep{du2022sharpness}, which accelerates the training process by replacing the sharpness measure with a trajectory loss.
However, this approach is heavily memory-consuming as it requires saving the output history of each example.
\par
In contrast to these studies, our approach concentrates on improving the training stability of sharpness-aware optimization, functioning as a plug-and-play component for SAM and its variants. 
Despite its simplicity, our approach is shown to be more robust with large learning rates and can achieve similar or even superior generalization performance compared to the vanilla SAM.
\section{Methodology}
\label{sec:methodology}
{
	While there exist some disputes \citep{dinh2017sharp, wen2024sharpness}, it is widely appreciated that flat minima empirically generalize better than sharp ones \citep{keskar2016large,chaudhari2019entropy,kaddour2022flat}.
	Motivated by this, SAM actively biases the training towards the flat regions of the loss landscape and seeks a neighborhood with low training losses.
	In practice, after a series of Taylor approximations, each SAM iterate can be decomposed into two steps,
	\begin{equation*}
		\begin{aligned}
			\boldsymbol{w}_t^{asc} = \boldsymbol{w}_t + \rho {\nabla F_{\Omega_t}(\boldsymbol{w}_t) }, \quad
			\boldsymbol{w}_{t+1} = \boldsymbol{w}_t - \eta\nabla F_{\Omega_t}(\boldsymbol{w}_t^{asc}),
		\end{aligned}
	\end{equation*}
	where $\Omega_t$ is a mini-batch of $S$ at $t$-th step, $\rho>0$ is the perturbation radius, and $\eta$ is the learning rate.
	By first ascending the weight along $\nabla F_{\Omega_t}(\boldsymbol{w}_t) $ and then descending it along $\nabla F_{\Omega_t}(\boldsymbol{w}_t^{asc})$, SAM penalizes the gradient norm \citep{zhao2022penalizing, compagnoni2023sde} and consistently minimizes the worse-case loss within the neighborhood, making the found solution more robust to distribution shift and consequently yielding a better generalization.
	\par
	In contrast to SGD, however, SAM faces a higher risk of getting trapped in the saddle points \citep{compagnoni2023sde, kim2023stability}, which may result in suboptimal outcomes \citep{du2017gradient, kleinberg2018alternative}.
	To gain some quantitative insights into how SAM exacerbates training stability,  let us consider the following function \citep{lucchi2021second},
	\begin{equation*}
		f(x_1, x_2) = \frac{1}{4} x_1^4 - x_1 x_2 + \frac{1}{2} x_2^2,
	\end{equation*}
	which has a strict saddle point at $(0, 0)$ and two global minima at $(-1, -1)$ and $(1, 1)$. 
	Given a random starting point, we want to know whether the training process can converge to one of the global minima.
	\begin{figure}[t]
		\centering
		\begin{subfigure}[b]{0.49\textwidth}
			\centering
			\includegraphics[width=\textwidth, clip, trim= 0 0 0 0]{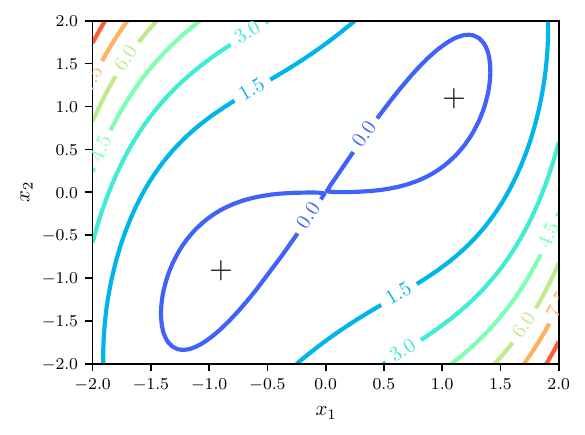}
			\caption{Loss landscape}
		\end{subfigure}
		\begin{subfigure}[b]{0.49\textwidth}
			\centering
			\includegraphics[width=\textwidth, clip, trim= 0 0 0 0]{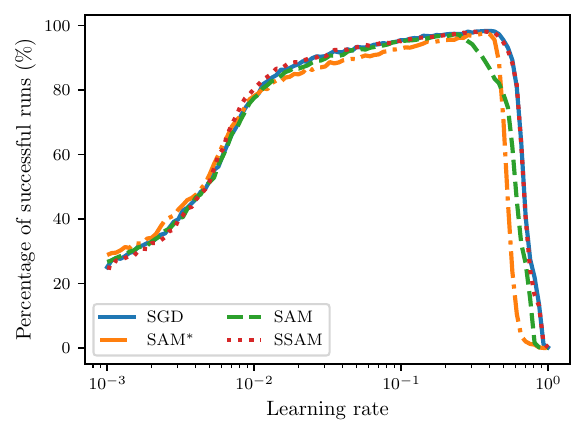}
			\caption{$\rho=0.05$}
		\end{subfigure}\\
		\begin{subfigure}[b]{0.49\textwidth}
			\centering
			\includegraphics[width=\textwidth, clip, trim= 0 0 0 0]{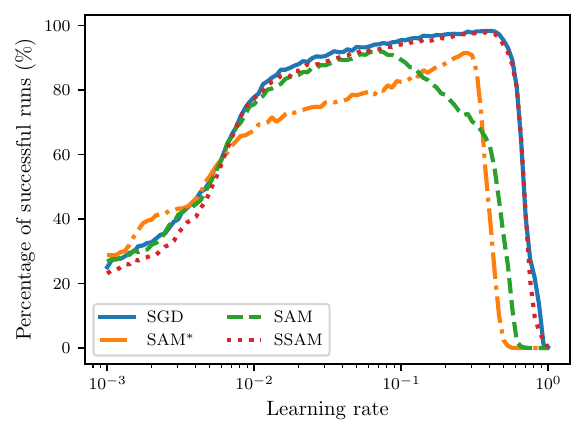}
			\caption{$\rho=0.2$}
		\end{subfigure}
		\begin{subfigure}[b]{0.49\textwidth}
			\centering
			\includegraphics[width=\textwidth, clip, trim= 0 0 0 0]{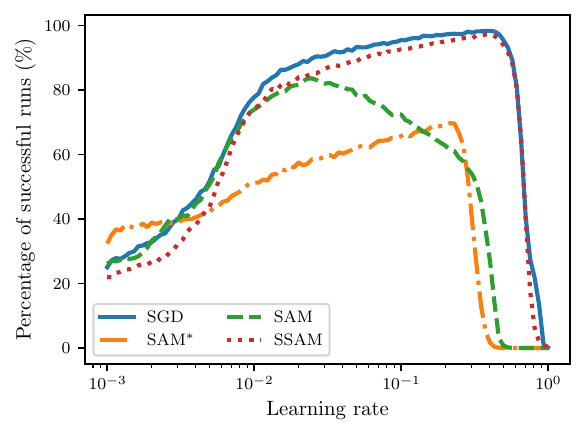}
			\caption{$\rho=0.4$}
		\end{subfigure}
		\caption{(a) Contour plot of  function $f(x_1, x_2) =  x_1^4/4 - x_1 x_2 + x_2^2/2$ and the symbol $(+)$ marks the global minima at $(-1, -1)$ and $(1, 1)$, respectively. (b) - (d) exhibit the rate of successful training as a function of the learning rate for different optimizers and perturbation radius $\rho$. Notice that the curve of SGD remains the same throughout these subplots since it does not depend on $\rho$.}
		\label{fig: toy example}
	\end{figure}
	\begin{algorithm}[t]
		\caption{SSAM Optimizer}
		\begin{algorithmic}[1]
			\renewcommand{\algorithmicrequire}{ \textbf{Input:}}
			\REQUIRE Training set $S=\{(x_i, y_i)\}_{i=1}^n$, objective function $F_S(\boldsymbol{w})$, initial weight $\boldsymbol{w}_0\in\mathbb{R}^d$, learning rate $\eta>0$, perturbation radius $\rho>0$, training iterations $T$, and base optimizer $\mathcal{A}$ (e.g. SGD)
			\renewcommand{\algorithmicensure}{ \textbf{Output:}}
			\ENSURE $\boldsymbol{w}_T$
			\FOR{$t=0, 1, \cdots, T-1$}
			\STATE Sample a mini-batch $\Omega_t=\{(x_{t_1}, y_{t_1}), \cdots, (x_{t_b}, y_{t_b})\}$;
			\STATE Compute gradient $\boldsymbol{g}_t=\nabla_{\boldsymbol{w}} F_{\Omega_t}(\boldsymbol{w})|_{\boldsymbol{w}=\boldsymbol{w}_t}$ of the loss over $\Omega_t$;
			\STATE Compute perturbed weight $\boldsymbol{w}^{asc}_t = \boldsymbol{w}_t + \rho \boldsymbol{g}_t$;
			\STATE Compute gradient $\boldsymbol{g}^{asc}_t=\nabla_{\boldsymbol{w}} F_{\Omega_t}(\boldsymbol{w})|_{\boldsymbol{w}=\boldsymbol{w}^{asc}_t}$ of the loss over the same $\Omega_t$;
			\STATE \textbf{Renormalize gradient as $\boldsymbol{g}^{asc}_t =  \frac{\|\boldsymbol{g}_t\|_2} { \|\boldsymbol{g}^{asc}_t\|_2}\boldsymbol{g}^{asc}_t$;}
			\STATE Update weight with base optimizer $\mathcal{A}$, e.g. $\boldsymbol{w}_{t+1} = \boldsymbol{w}_t - \eta \boldsymbol{g}^{asc}_t$;
			\ENDFOR
		\end{algorithmic}
		\label{algo: stablesam v1}
	\end{algorithm}
	\par
	For this purpose, we select 100 different learning rates that are equispaced between 0.001 and 0.3 on the logarithm scale. For each learning rate, we then uniformly sample 10000 random points from the square $[-2, 2]\times [-2, 2]$ and report the total percentage of runs that eventually converge to the global minima.
	We mark the runs that get stuck in the saddle point or fail to converge as unsuccessful runs.
	To introduce stochasticity during training, we manually perturb the gradient with zero-mean Gaussian noise with a variance of $0.005$.
	As shown in Figure \ref{fig: toy example}, the failed percentage of SAM and  $\mathrm{SAM}^\ast$ first blow up when we gradually increase the learning rate, suggesting that a smaller learning rate is necessary for sharpness-aware training to ensure convergence.
	Moreover, we can observe that SGD always achieves the highest rate of successful training, while $\mathrm{SAM}^\ast$ is the most unstable optimizer.
	Notice that the stability of sharpness-aware training also heavily relies on the perturbation radius $\rho$.
	Often, a larger $\rho$ corresponds to a lower percentage of successful runs.
	This indicates that both SAM and $\mathrm{SAM}^\ast$ become more and more difficult to escape from the saddle point $(0, 0)$.
	\par
	To address this issue, we propose a simple strategy, dubbed SSAM, to improve the stability of sharpness-aware training.
	As shown in Algorithm \ref{algo: stablesam v1}\footnote{A PyTorch implementation is available at \url{https://github.com/cltan023/stablesam2024}.}, the only difference from SAM is that we include an extra renormalization step (line 6) to ensure that the gradient norm of the descent step does not exceed that of the ascent step.
	The ratio, $\gamma_t=\|\nabla F_{\Omega_t} (\boldsymbol{w}_t)\|_2 / \|\nabla F_{\Omega_t} (\boldsymbol{w}_t^{asc})\|_2$, which we refer to as the \emph{renormalization factor}, can be interpreted as follows.
	When $\|\nabla F_{\Omega_t} (\boldsymbol{w}_t^{asc})\|_2$ is larger than $\|\nabla F_{\Omega_t} (\boldsymbol{w}_t)\|_2$,  we downscale the norm of $\nabla F_{\Omega_t} (\boldsymbol{w}_t^{asc})$ to ensure that the iterates move in a smaller step towards the flat regions and thus we can reduce the chance of fluctuation and divergence.
	In contrast, when $\|\nabla F_{\Omega_t} (\boldsymbol{w}_t^{asc})\|_2$ is smaller than $\|\nabla F_{\Omega_t} (\boldsymbol{w}_t)\|_2$, a situation that may occur near the saddle points, we upscale the norm of $\nabla F_{\Omega_t} (\boldsymbol{w}_t^{asc})$ to incur a larger perturbation to improve the escaping efficiency.
	\par
	After applying the renormalization strategy, as shown in Figure \ref{fig: toy example}, we can observe that this issue can be remedied to a large extent because the curve of SSAM now remains approximately the same as SGD even for large learning rates.
	Similar results for realistic neural networks can also be found in Appendix \ref{sec: Training Instability on Realistic Neural Networks}.
	It should be clarified that the analysis here is not from the generalization perspective, but instead from the optimization perspective only.
	Compared to SAM, our approach does not introduce any additional hyperparameter so that it can be integrated with other SAM variants almost at no computational cost.
}
\section{Theoretical Analysis}
\label{sec:theoretical analysis}
The generalization ability of sharpness-aware training was initially studied by the PAC-Bayesian theory \citep{foret2020sharpness,yue2023sharpness,zhuang2021surrogate}.
This approach, however, is fundamentally limited since the generalization bound is focused on the worst-case perturbation rather than the realistic one-step ascent approximation \citep{wen2022does}.
For a certain class of problems, an analysis from the perspective of implicit bias suggests that SAM can always choose a better solution than SGD \citep{andriushchenko2022towards}.
In the small learning rate regime, \citet{compagnoni2023sde} further characterized the continuous-time models for SAM in the form of a stochastic differential equation and concluded that SAM is attracted to saddle points under some realistic conditions, an observation which has also been unveiled by \citet{kim2023stability}.
Moreover, \citet{bartlett2023dynamics} argued that SAM converges to a cycle that oscillates between the minimum along the principal direction of the Hessian of the loss function.
Different from these studies, here we investigate the generalization performance of SAM via algorithmic stability \citep{bousquet2002stability, hardt2016train} and together with its convergence properties present an upper bound over its expected excess risk.
We first show that SAM consistently generalizes better than SGD, though a much smaller learning rate is required.
Finally, we show how our proposed method, SSAM, extends the regime of learning rate and can achieve a better generalization performance than SAM.

\subsection{Notations and Preliminaries}
Let $X\subset\mathbb{R}^p$ and $Y\subset\mathbb{R}$ denote the feature and label space, respectively.
We consider a training set $S$ of $n$ examples, each of which is randomly sampled from an unknown distribution $\mathfrak{D}$ over the data space $Z = X \times Y$.
Given a learning algorithm $\mathcal{A}$, it learns a hypothesis that relates the input $x\in X$ to the output $y\in Y$.
For deep neural networks, the learned hypothesis is parameterized by the network parameters $\boldsymbol{w}\in\mathbb{R}^d$.
\par
Suppose $f(\boldsymbol{w}, z): \mathbb{R}^d\times Z \mapsto \mathbb{R}_{+}$ is a non-negative cost function, we then can define the \emph{population risk}
\begin{equation*}
	F_{\mathfrak{D}}(\boldsymbol{w}) = \mathbb{E}_{z\sim\mathfrak{D}}\left[f(\boldsymbol{w}, z)\right],
\end{equation*}
and the \emph{empirical risk}
\begin{equation*}
	F_S(\boldsymbol{w}) = \frac{1}{n}\sum_{i=1}^n f(\boldsymbol{w}, z_i).
\end{equation*}
In practice, we cannot compute $F_{\mathfrak{D}}(\boldsymbol{w})$ directly since the data distribution $\mathfrak{D}$ is unknown.
However, once the training set $S$ is given, we have access to its estimation and can minimize the empirical risk $F_S(\boldsymbol{w})$ instead, a process which is often referred to as \emph{empirical risk minimization}.
Let $\boldsymbol{w}_{\mathcal{A}, S}$ be the output returned by minimizing the empirical risk $F_S(\boldsymbol{w})$ with learning algorithm $\mathcal{A}$, and $\boldsymbol{w}_{\mathfrak{D}}^*$ be one minimizer of the population risk $F_{\mathfrak{D}}(\boldsymbol{w})$, namely, $\boldsymbol{w}_{\mathfrak{D}}^* \in\arg\min_{\boldsymbol{w}} F_{\mathfrak{D}}(\boldsymbol{w})$.
Since $\boldsymbol{w}_{\mathcal{A}, S}$ in high probability will not be the same with $\boldsymbol{w}_{\mathfrak{D}}^*$, we are interested in how far $\boldsymbol{w}_{\mathcal{A}, S}$ deviates from $\boldsymbol{w}_{\mathfrak{D}}^*$ when evaluated on an unseen example $z\sim\mathfrak{D}$.
\par
A natural measure to quantify this difference is the so-called \emph{expected excess risk}, 
\begin{equation*}
	\begin{aligned}
		\varepsilon_{exc} &= \mathbb{E}\left[F_{\mathfrak{D}}(\boldsymbol{w}_{\mathcal{A}, S}) - F_{\mathfrak{D}}(\boldsymbol{w}_{\mathfrak{D}}^*)\right] \\
		& = \underbrace{\mathbb{E}\left[F_{\mathfrak{D}}(\boldsymbol{w}_{\mathcal{A}, S}) - F_S(\boldsymbol{w}_{\mathcal{A}, S})\right]}_{\varepsilon_{gen}} + \underbrace{\mathbb{E}\left[F_S(\boldsymbol{w}_{\mathcal{A}, S}) - F_S(\boldsymbol{w}_S^*)\right]}_{\varepsilon_{opt}} + \underbrace{\mathbb{E}\left[F_S(\boldsymbol{w}_S^*) - F_{\mathfrak{D}}(\boldsymbol{w}_{\mathfrak{D}}^*)\right]}_{\varepsilon_{approx}},
	\end{aligned}
\end{equation*}
where $\boldsymbol{w}_S^* \in \arg\min_{\boldsymbol{w}} F_S(\boldsymbol{w})$.
Since $\boldsymbol{w}_{\mathfrak{D}}^*$ remains constant for the population risk $F_{\mathfrak{D}}(\boldsymbol{w})$ which depends only on the data distribution and loss function, it follows that the \emph{expected approximation error} $\varepsilon_{approx}=\mathbb{E}\left[F_S(\boldsymbol{w}_S^*) - F_{\mathfrak{D}}(\boldsymbol{w}_{\mathfrak{D}}^*)\right]=\mathbb{E}\left[F_S(\boldsymbol{w}_S^*) - F_{S}(\boldsymbol{w}_{\mathfrak{D}}^*)\right]\leq 0$.
Therefore, it often suffices to obtain tight control of the \emph{expected excess risk} $\varepsilon_{exc}$ by bounding the \emph{expected generalization error}\footnote{It is worth noting that the difference between the test error and the training error in some literature is referred to as \emph{generalization gap} and the test error alone goes by \emph{generalization error}.} $\varepsilon_{gen}$ and the \emph{expected optimization error} $\varepsilon_{opt}$.
\par
For learning algorithms based on iterative optimization, $\varepsilon_{opt}$ in many cases can be analyzed via a convergence analysis \citep{bubeck2015convex}.
Meanwhile, to derive an upper bound over $\varepsilon_{gen}$, we can use the following theorem, which is due to \citet{hardt2016train}, indicating that the generalization error could be bounded via the uniform stability \citep{bousquet2002stability}.
Indeed, the uniform stability characterizes how sensitive the output of the learning algorithm $\mathcal{A}$ is when a single example in the training set $S$ is modified.
\begin{theorem}[Generalization error under $\varepsilon$-uniformly stability]
	Let $S$ and $S^\prime$ denote two training sets i.i.d.~sampled from the same data distribution $\mathfrak{D}$ such that $S$ and $S^\prime$ differ in at most one example. A learning algorithm $\mathcal{A}$ is $\varepsilon$-uniformly stable if and only if for all samples $S$ and $S^\prime$, the following inequality holds
	\begin{equation*}
		\sup_{z} \mathbb{E}|f(\boldsymbol{w}_{\mathcal{A}, S}, z) - f(\boldsymbol{w}_{\mathcal{A}, S^\prime}, z)| \leq \varepsilon.
	\end{equation*}
	Furthermore, if $\mathcal{A}$ is $\varepsilon$-uniformly stable, the expected generalization error $\varepsilon_{gen}$ is upper bounded by $\varepsilon$, namely,
	\begin{equation*}
		\mathbb{E}\left[F_{\mathfrak{D}}(\boldsymbol{w}_{\mathcal{A}, S}) - F_S (\boldsymbol{w}_{\mathcal{A}, S})\right] \leq \varepsilon.
	\end{equation*}
\end{theorem}
To ease notation, we use $f(\boldsymbol{w})$ interchangeably with $f(\boldsymbol{w}, z)$ in the sequel as long as it is clear from the context that $z$ is being held constant or can be understood from prior information.
\subsection{Expected Excess Risk Analysis of SAM}
In this section, we first investigate the stability of SAM and then its convergence property, together yielding an upper bound over the expected excess risk $\varepsilon_{exc}$.
We restrict our attention to the strongly convex case so that we can compare against known results, particularly from \citet{hardt2016train}.
\subsubsection{Stability}
Consider the optimization trajectories $\boldsymbol{w}_0, \boldsymbol{w}_1, \cdots, \boldsymbol{w}_T$ and $\boldsymbol{v}_0, \boldsymbol{v}_1, \cdots, \boldsymbol{v}_T$ induced by running SAM for $T$ steps on sample $S$ and $S^\prime$, which differ from each other only by one example.
Suppose that the loss function $f(\boldsymbol{w}, z)$ is $G$-Lipschitz with respect to the first argument, then it holds for all $z \in Z$ that
\begin{equation}
	\label{eq: stability lipschitz inequality}
	| f(\boldsymbol{v}_T, z) - f(\boldsymbol{w}_T, z)   | \leq G \| \boldsymbol{v}_T - \boldsymbol{w}_T   \|_2.
\end{equation}
Therefore, the remaining step in our setup is to upper bound $ \| \boldsymbol{v}_T - \boldsymbol{w}_T   \|_2$, which can be recursively controlled by the growth rate.
Since $S$ and $S^\prime$ differ in only one example, at every step $t$, the selected examples from $S$ and $S^\prime$, say $z$ and $z^\prime$, are either the same or not.
In the lemma below, we show that $ \|\boldsymbol{v}_t - \boldsymbol{w}_t \|_2$ is contracting when $z$ and $z^\prime$ are the same. 
\begin{lemma}
	\label{lemma: sam same example}
	Assume that the per-example loss function $f(\boldsymbol{w}, z)$ is $\mu$-strongly convex, $L$-smooth, and $G$-Lipschitz continuous with respect to the first argument $\boldsymbol{w}$. Suppose that at step $t$, the examples selected by SAM are the same in $S$ and $S^\prime$ and the update rules are denoted by
	$\boldsymbol{w}_{t+1} = \boldsymbol{w}_t - \eta \nabla f(\boldsymbol{w}_t^{asc}, z)$ and $\boldsymbol{v}_{t+1} = \boldsymbol{v}_t - \eta \nabla f(\boldsymbol{v}_t^{asc}, z)$, respectively.
	Then,  it follows that
	\begin{equation}
		\label{eq: sam same example}
		\|\boldsymbol{v}_{t+1} - \boldsymbol{w}_{t+1} \|_2 \leq \left(1 -  \left(1+\mu\rho\right) \frac{\eta\mu L}{\mu+L}\right)  \|\boldsymbol{v}_{t} - \boldsymbol{w}_{t} \|_2,
	\end{equation}
	where the learning rate $\eta$ satisfies that
	\begin{equation}
		\label{eq: learning rate constraint}
		\eta \leq \frac{2}{\mu + L} - \frac{\mu + L}{2\mu L (\mu/\rho L^2 + 1)}.
	\end{equation}
\end{lemma}
\begin{proof}
	To prove this result, we first would like to lower bound the term $ \|\boldsymbol{v}_{t}^{asc} - \boldsymbol{w}_t^{asc} \|_2^2$ as follows:
	\begin{equation*}
		\begin{aligned}
			\|\boldsymbol{v}_{t}^{asc} - \boldsymbol{w}_t^{asc} \|_2^2 &=  \|\boldsymbol{v}_{t} - \boldsymbol{w}_{t} \|_2^2 + 2\rho  \left<\boldsymbol{v}_{t} - \boldsymbol{w}_{t}, \nabla f(\boldsymbol{v}_t)-\nabla f(\boldsymbol{w}_t) \right> + \rho^2  \|\nabla f(\boldsymbol{v}_{t}) - \nabla f(\boldsymbol{w}_{t}) \|_2^2 \\
			&\geq (1 + 2\mu\rho)  \|\boldsymbol{v}_{t} - \boldsymbol{w}_{t} \|_2^2 + \rho^2  \|\nabla f(\boldsymbol{v}_{t}) - \nabla f(\boldsymbol{w}_{t}) \|_2^2 \\
			&\geq (1 + \mu\rho)  \|\boldsymbol{v}_{t} - \boldsymbol{w}_{t} \|_2^2 + ({\mu\rho}/{L^2} + \rho^2)  \|\nabla f(\boldsymbol{v}_{t}) - \nabla f(\boldsymbol{w}_{t}) \|_2^2.
		\end{aligned}
	\end{equation*}
	\par\noindent
	According to the update rule, we further have
	\begin{align*}
		& \|\boldsymbol{v}_{t+1} - \boldsymbol{w}_{t+1} \|_2^2 
		={}  \| \boldsymbol{v}_t - \eta\nabla f(\boldsymbol{v}_t^{asc}) -\left(\boldsymbol{w}_t - \eta\nabla f(\boldsymbol{w}_t^{asc})\right)  \|_2^2 \\
		&{}={}  \|\boldsymbol{v}_{t} - \boldsymbol{w}_{t} \|_2^2 - 2\eta \left<\boldsymbol{v}_{t} - \boldsymbol{w}_{t}, \nabla f(\boldsymbol{v}_t^{asc})-\nabla f(\boldsymbol{w}_t^{asc}) \right> + \eta^2 \|\nabla f(\boldsymbol{v}_t^{asc})-\nabla f(\boldsymbol{w}_t^{asc}) \|_2^2 \\
		& =\begin{aligned}[t]
			& \|\boldsymbol{v}_{t} - \boldsymbol{w}_{t} \|_2^2 - 2\eta \left<\boldsymbol{v}_{t}^{asc} - \boldsymbol{w}_{t}^{asc}, \nabla f(\boldsymbol{v}_t^{asc})-\nabla f(\boldsymbol{w}_t^{asc}) \right>\\ 
			&\quad+ 2\rho\eta \left<\nabla f(\boldsymbol{v}_{t}) - \nabla f(\boldsymbol{w}_{t}), \nabla f(\boldsymbol{v}_t^{asc})-\nabla f(\boldsymbol{w}_t^{asc}) \right> + \eta^2 \|\nabla f(\boldsymbol{v}_t^{asc})-\nabla f(\boldsymbol{w}_t^{asc}) \|_2^2 
		\end{aligned}\\
		&\stackrel{\textcircled{1}}{\leq} \begin{aligned}[t]
			&\left(1 - 2 \left(1+\mu \rho \right) \frac{\eta\mu L}{\mu+L}\right)  \|\boldsymbol{v}_{t} - \boldsymbol{w}_{t} \|_2^2 - 2\left(\frac{\mu\rho}{L^2} + \rho^2\right)\frac{\eta\mu L}{\mu+L} \|\nabla f(\boldsymbol{v}_t)-\nabla f(\boldsymbol{w}_t) \|_2^2\\
			&\quad+2\rho\eta \left<\nabla f(\boldsymbol{v}_{t}) - \nabla f(\boldsymbol{w}_{t}), \nabla f(\boldsymbol{v}_t^{asc})-\nabla f(\boldsymbol{w}_t^{asc}) \right> + \left(\eta^2- \frac{2\eta}{\mu + L}\right) \|\nabla f(\boldsymbol{v}_t^{asc})-\nabla f(\boldsymbol{w}_t^{asc}) \|_2^2 
		\end{aligned}\\
		&\stackrel{\textcircled{2}}{\leq} \begin{aligned}[t]
			&\left(1 - 2 \left(1+\mu \rho \right) \frac{\eta\mu L}{\mu+L}\right)  \|\boldsymbol{v}_{t} - \boldsymbol{w}_{t} \|_2^2 + \left[\frac{\rho^2\eta}{\frac{2}{\mu+L} - \eta} - 2\left(\frac{\mu\rho}{L^2} + \rho^2\right)\frac{\eta\mu L}{\mu+L}\right] \|\nabla f(\boldsymbol{v}_t)-\nabla f(\boldsymbol{w}_t) \|_2^2\\
			&\quad+ \left(\eta^2- \frac{2\eta}{\mu + L}\right)\left[\left(\nabla f\left(\boldsymbol{v}_t^{asc}\right)-\nabla f\left(\boldsymbol{w}_t^{asc}\right)\right) - \frac{\rho}{\frac{2}{\mu+L} - \eta}\left(\nabla f\left(\boldsymbol{v}_{t}\right) - \nabla f\left(\boldsymbol{w}_{t}\right)\right) \right]^2
		\end{aligned}\\
		&\stackrel{\textcircled{3}}{\leq}{}\left(1 - 2 \left(1+\mu\rho\right) \frac{\eta\mu L}{\mu+L}\right)  \|\boldsymbol{v}_{t} - \boldsymbol{w}_{t} \|_2^2,
	\end{align*}
	where $\textcircled{1}$ is due to the coercivity of the loss function (cf. Appendix \ref{appendix: coeveritty of strongly convex})  that
	\begin{equation*}
		\left<\nabla f(\boldsymbol{v}_t^{asc})-\nabla f(\boldsymbol{w}_t^{asc}), \boldsymbol{v}_{t}^{asc} - \boldsymbol{w}_t^{asc} \right> \geq \frac{\mu L}{\mu + L}  \|\boldsymbol{v}_{t}^{asc} - \boldsymbol{w}_t^{asc} \|_2^2 + \frac{1}{\mu + L} \|\nabla f(\boldsymbol{v}_t^{asc})-\nabla f(\boldsymbol{w}_t^{asc}) \|_2^2.
	\end{equation*}
	Moreover, \textcircled{3} holds since the last two terms of \textcircled{2} are smaller than zero provided that the learning rate $\eta$ satisfies the given condition.
	Consequently, we have
	\begin{equation*}
		\|\boldsymbol{v}_{t+1} - \boldsymbol{w}_{t+1} \|_2 \leq \left(1 - 2 \left(1+\mu\rho\right) \frac{\eta\mu L}{\mu+L}\right)^{1/2}  \|\boldsymbol{v}_{t} - \boldsymbol{w}_{t} \|_2\leq \left(1 -  \left(1+\mu\rho\right) \frac{\eta\mu L}{\mu+L}\right)  \|\boldsymbol{v}_{t} - \boldsymbol{w}_{t} \|_2,
	\end{equation*}
	where the last inequality is due to the fact that $\sqrt{1-x}\leq 1-x/2$ holds for all $x\in[0, 1]$.
\end{proof}
\begin{remark}
	To ensure that the learning rate $\eta$ is feasible, the right-hand side of \eqref{eq: learning rate constraint} should be at least larger than zero.
	This holds for any perturbation radius $\rho >0 $ if $\mu = L$.
	However, if $\mu < L$, we further need to require that $\rho < 4\mu^2/L(L-\mu)^2$.
	It is also worth noting that the following inequality holds for all $\rho > 0$
	\begin{equation*}
		\frac{2}{\mu + L} - \frac{\mu + L}{2\mu L (\mu/\rho L^2 + 1)} < \frac{2}{(1 +\mu\rho)(\mu + L)},
	\end{equation*}
	implying that the contractivity of \eqref{eq: sam same example} can be guaranteed.
\end{remark}
On the other hand, with probability $1/n$, the examples selected by SAM, say $z$ and $z^\prime$, are different in both $S$ and $S^\prime$.
In this case, we can simply bound the growth in $ \|\boldsymbol{v}_t - \boldsymbol{w}_t \|_2$ by the norms of $\nabla f(\boldsymbol{w}, z)$ and $\nabla f(\boldsymbol{v}, z^\prime)$.
\begin{lemma}
	\label{lemma: sam different example}
	Assume the same settings as in Lemma \ref{lemma: sam same example}.
	For the $t$-th iteration, suppose that the examples selected by SAM are different in $S$ and $S^\prime$ and the update rules are denoted by $\boldsymbol{w}_{t+1} = \boldsymbol{w}_t - \eta \nabla f(\boldsymbol{w}_t^{asc}, z)$ and $\boldsymbol{v}_{t+1} = \boldsymbol{v}_t - \eta \nabla f(\boldsymbol{v}_t^{asc}, z^\prime)$, respectively. Consequently, we have
	\begin{equation*}
		\|\boldsymbol{v}_{t+1} - \boldsymbol{w}_{t+1} \|_2 \leq \left(1 -  \left(1+\mu\rho\right) \frac{\eta\mu L}{\mu+L}\right)  \|\boldsymbol{v}_{t} - \boldsymbol{w}_{t} \|_2 + 2\eta G.
	\end{equation*}
\end{lemma}
\begin{proof}
	The proof is straightforward. It follows immediately
	\begin{equation*}
		\begin{aligned}
			\|\boldsymbol{v}_{t+1} - \boldsymbol{w}_{t+1} \|_2  & =  \|\boldsymbol{v}_t - \eta\nabla f(\boldsymbol{v}_t^{asc}, z^\prime) - \left(\boldsymbol{w}_t - \eta\nabla f(\boldsymbol{w}_t^{asc}, z^\prime)\right) - \eta\left(\nabla f(\boldsymbol{w}_t^{asc}, z^\prime) - \nabla f(\boldsymbol{w}_t^{asc}, z)\right) \|_2  \\
			& \leq  \|\boldsymbol{v}_t - \eta\nabla f(\boldsymbol{v}_t^{asc}, z^\prime) - \left(\boldsymbol{w}_t - \eta\nabla f(\boldsymbol{w}_t^{asc}, z^\prime)\right)  \|_2 + \eta  \|\nabla f(\boldsymbol{w}_t^{asc}, z^\prime) - \nabla f(\boldsymbol{w}_t^{asc}, z) \|_2 \\
			&\leq \left(1 -  \left(1+\mu\rho\right) \frac{\eta\mu L}{\mu+L}\right)  \|\boldsymbol{v}_{t} - \boldsymbol{w}_{t} \|_2 + 2\eta G,
		\end{aligned}
	\end{equation*}
	where the last inequality comes from Lemma \ref{lemma: sam same example}.
\end{proof}
With the above two lemmas, we are now ready to give an upper bound over the expected generalization error of SAM.
\begin{theorem}
	\label{theorem: sam stability}
	Assume that the per-example loss function $f(\boldsymbol{w}, z)$ is $\mu$-strongly convex, $L$-smooth, and $G$-Lipschitz continuous with respect to the first argument $\boldsymbol{w}$.
	Suppose we run the SAM iteration with a constant learning rate $\eta$ satisfying \eqref{eq: learning rate constraint} for $T$ steps.
	Then,  SAM satisfies uniform stability with
	\begin{equation*}
		\varepsilon_{\mathrm{gen}}^{\mathrm{sam}} \leq \frac{2G^2(\mu + L)}{n \mu L(1+\mu\rho)}\left\{1 -\left[1- \left(1+\mu\rho\right)\frac{\eta \mu L}{\mu + L}\right]^T\right\}.
	\end{equation*}
\end{theorem}
\begin{proof}
	Define $\delta_t=  \|\boldsymbol{w}_t - \boldsymbol{v}_t \|_2$ to denote the Euclidean distance between $\boldsymbol{w}_t$ and $\boldsymbol{v}_t$ as training progresses.
	Observe that at any step $t\leq T$, with a probability $1-1/n$, the selected examples from $S$ and $S^\prime$ are the same.
	In contrast, with a probability of $1/n$, the selected examples are different. This is because $S$ and $S^\prime$ only differ by one example. Therefore, from Lemmas \ref{lemma: sam same example} and \ref{lemma: sam different example}, we conclude that
	\begin{equation*}
		\begin{aligned}
			\mathbb{E}[\delta_t] &\leq \left(1-\frac{1}{n}\right)\left(1 -  \left(1+\mu\rho\right) \frac{\eta\mu L}{\mu+L}\right)\mathbb{E}[\delta_{t-1}] + \frac{1}{n}\left(1 -  \left(1+\mu\rho\right) \frac{\eta\mu L}{\mu+L}\right)\mathbb{E}[\delta_{t-1}] + \frac{2\eta G}{n} \\
			&= \left(1 -  \left(1+\mu\rho\right) \frac{\eta\mu L}{\mu+L}\right)\mathbb{E}[\delta_{t-1}] + \frac{2\eta G}{n}.
		\end{aligned}
	\end{equation*}
	Unraveling the above recursion yields
	\begin{equation*}
		\begin{aligned}
			\mathbb{E}[\delta_T] \leq \frac{2\eta G}{n}\sum_{t=0}^{T-1}\left(1 -  \left(1+\mu\rho\right)\frac{\eta \mu L}{\mu + L}\right)^t = \frac{2G(\mu + L)}{n \mu L(1+\mu\rho)}\left\{1 -\left[1- \left(1+\mu\rho\right)\frac{\eta \mu L}{\mu + L}\right]^T\right\}.
		\end{aligned}
	\end{equation*}
	Plugging this inequality into (\ref{eq: stability lipschitz inequality}), we complete the proof.
\end{proof}
In the same strongly convex setting, it is known that SGD allows for a larger learning rate (namely, $\eta \leq \frac{2}{\mu + L}$) to attain a similar generalization bound \citep[Lemma 3.7]{hardt2016train}.
However, when both SGD and SAM use a constant learning rate satisfying \eqref{eq: learning rate constraint}, the following corollary suggests that SAM consistently generalizes better than SGD.
\begin{corollary}
	Assume the same settings as in Theorem \ref{theorem: sam stability}.
	Suppose we run SGD and SAM with a constant learning rate $\eta$ satisfying \eqref{eq: learning rate constraint} for $T$ steps.
	Then, SAM consistently achieves a tighter generalization bound than SGD.
\end{corollary}
\begin{proof}
	Following \citet[Theorem 3.9]{hardt2016train}, we can derive a similar generalization bound for SGD as follows
	\begin{equation*}
		\varepsilon_{\mathrm{gen}}^{\mathrm{sgd}} \leq \frac{2G^2(\mu + L)}{n \mu L}\left\{1 -\left[1- \frac{\eta \mu L}{\mu + L}\right]^T\right\}.
	\end{equation*}
	Define $q(x) = a(1-x)^T - (1-ax)^T$, where $a = 1 + \mu\rho$ and $x=\frac{\eta \mu L}{\mu + L}$.
	Note that $a>1$ and $0<ax < 1$. With a simple calculation, we have
	\begin{equation*}
		q^\prime(x) = aT\left[\left(1-ax\right)^{T-1} - \left(1 -x \right)^{T-1}\right],
	\end{equation*}
	implying that $q^\prime(x) \leq 0$ for any $T\geq 1$ and as a result we have $q(x) \leq a - 1$.
	Then, it follows that
	\begin{equation*}
		\varepsilon_{\mathrm{gen}}^{\mathrm{sam}} = \frac{2G^2(\mu + L)}{n \mu L(1+\mu\rho)}\left\{1 -\left[1- \left(1+\mu\rho\right)\frac{\eta \mu L}{\mu + L}\right]^T\right\} \leq \frac{2G^2(\mu + L)}{n \mu L}\left\{1 -\left[1- \frac{\eta \mu L}{\mu + L}\right]^T\right\} = \varepsilon_{\mathrm{gen}}^{\mathrm{sgd}},
	\end{equation*}
	thus concluding the proof.
\end{proof}
\subsubsection{Convergence}
From the perspective of convergence, we can further prove that SAM converges to a noisy ball if the learning rate $\eta$ is fixed.
Let $z_t$ be the example that is chosen by SAM at $t$-th step and $\nabla f(\boldsymbol{w}_t^{asc}) = \nabla f\left(\boldsymbol{w}_t + \rho \nabla f\left(\boldsymbol{w}_t, z_t\right), z_t\right)$ be the stochastic gradient of the descent step.
It is worth noting that the same example $z_t$ is used in the ascent and descent steps.
The following lemma shows that $\nabla f(\boldsymbol{w}_t^{asc})$ may not be well-aligned with the full-batch gradient $\nabla F_S(\boldsymbol{w}_t)$.
\begin{lemma}
	\label{lemma: sam descent alignment}
	Assume the loss function $f(\boldsymbol{w}, z)$ is $\mu$-strongly convex, $L$-smooth, and $G$-Lipschitz continuous with respect to the first argument $\boldsymbol{w}$.
	Then, we have for all $\boldsymbol{w}_t\in\mathbb{R}^d$, 
	\begin{equation*}
		\mathbb{E} \left<\nabla f(\boldsymbol{w}_t^{asc}), \nabla F_S(\boldsymbol{w}_t)\right> \geq \rho(\mu + L)\left\|\nabla F_S(\boldsymbol{w}_t)\right\|_2^2 - \frac{\rho^2L^2G^2}{2}.
	\end{equation*}
\end{lemma}
\begin{proof}
	First, it is easy to check that $F_S(\boldsymbol{w})$ is $\mu$-strongly convex, $L$-smooth, and $G$-Lipschitz continuous with respect to the first argument $\boldsymbol{w}$ as well.
	Let $\widehat{\boldsymbol{w}}_t^{asc} = \boldsymbol{w}_t + \rho \nabla F_S(\boldsymbol{w}_t)$, we have
	\begin{equation*}
		\begin{aligned}
			\left<\nabla f(\boldsymbol{w}_t^{asc}) - \nabla f(\widehat{\boldsymbol{w}}_t^{asc}), \nabla F_S(\boldsymbol{w}_t)\right> 
			&\leq \frac{1}{2}\left\|\nabla f(\boldsymbol{w}_t^{asc}) - \nabla f(\widehat{\boldsymbol{w}}_t^{asc})\right\|_2^2 + \frac{1}{2}\left\|\nabla F_S(\boldsymbol{w}_t)\right\|_2^2 \\
			& \leq \frac{\rho^2L^2}{2}\left\|\nabla f(\boldsymbol{w}_t) - \nabla F_S(\boldsymbol{w}_t)\right\|_2^2 + \frac{1}{2}\left\|\nabla F_S(\boldsymbol{w}_t)\right\|_2^2.
		\end{aligned}
	\end{equation*}
	After taking the expectation, it follows that
	\begin{equation*}
		\mathbb{E} \left<\nabla f(\boldsymbol{w}_t^{asc}) - \nabla f(\widehat{\boldsymbol{w}}_t^{asc}), \nabla F_S(\boldsymbol{w}_t)\right> \leq \frac{\rho^2L^2G^2}{2} + \frac{1-\rho^2L^2}{2}\left\|\nabla F_S(\boldsymbol{w}_t)\right\|_2^2.
	\end{equation*}
	On the other hand, 
	\begin{equation*}
		\begin{aligned}
			\mathbb{E}\left<\nabla f(\widehat{\boldsymbol{w}}_t^{asc}), \nabla F_S(\boldsymbol{w}_t)\right>
			& = \left<\nabla F_S(\widehat{\boldsymbol{w}}_t^{asc}), \nabla F_S(\boldsymbol{w}_t)\right> \\
			& = \left<\nabla F_S(\widehat{\boldsymbol{w}}_t^{asc}) - \nabla F_S(\boldsymbol{w}_t), \nabla F_S(\boldsymbol{w}_t)\right> + \left\|\nabla F_S(\boldsymbol{w}_t)\right\|_2^2 \\
			& = \frac{1}{\rho}\left<\nabla F_S(\boldsymbol{w}+\rho\nabla F_S(\boldsymbol{w}_t)) - \nabla F_S(\boldsymbol{w}_t), \rho\nabla F_S(\boldsymbol{w}_t)\right> + \left\|\nabla F_S(\boldsymbol{w}_t)\right\|_2^2 \\
			& \geq \left(1 + \mu\rho\right)\left\|\nabla F_S(\boldsymbol{w}_t)\right\|_2^2.
		\end{aligned}
	\end{equation*}
	Combining the above results, we have
	\begin{equation*}
		\begin{aligned}
			\mathbb{E}\left<\nabla f(\boldsymbol{w}_t^{asc}), \nabla F_S(\boldsymbol{w}_t)\right> 
			&= \mathbb{E}\left<\nabla f(\boldsymbol{w}_t^{asc}) - \nabla f(\widehat{\boldsymbol{w}}_t^{asc}), \nabla F_S(\boldsymbol{w}_t)\right> + \mathbb{E}\left<\nabla f(\widehat{\boldsymbol{w}}_t^{asc}), \nabla F_S(\boldsymbol{w}_t)\right>\\
			&\geq \left(\frac{1+\rho^2L^2}{2} + \mu\rho \right)\left\|\nabla F_S(\boldsymbol{w}_t)\right\|_2^2 - \frac{\rho^2L^2G^2}{2} \\
			& \geq \rho(\mu + L)\left\|\nabla F_S(\boldsymbol{w}_t)\right\|_2^2 - \frac{\rho^2L^2G^2}{2},
		\end{aligned}
	\end{equation*}
	completing the proof.
\end{proof}
\begin{theorem}
	\label{theorem: sam convergence}
	Assume that the per-example loss function $f(\boldsymbol{w}, z)$ is $\mu$-strongly convex, $L$-smooth and $G$-Lipschitz continuous with respect to the first argument $\boldsymbol{w}$.
	Consider the sequence $\boldsymbol{w}_0, \boldsymbol{w}_1, \cdots, \boldsymbol{w}_T$ generated by running SAM with a constant learning rate $\eta$ for $T$ steps. 
	Let $\boldsymbol{w}^* \in \arg \inf_{\boldsymbol{w}} F_S(\boldsymbol{w})$, it follows that
	\begin{equation*}
		\begin{aligned}
			\varepsilon_{\mathrm{opt}}^{\mathrm{sam}} = \mathbb{E} \left[F_S(\boldsymbol{w}_T)-F_S(\boldsymbol{w}^{*})\right] 
			\leq \left[1 - 2\eta\mu\rho(\mu + L)\right]^T \mathbb{E}\left[F_S(\boldsymbol{w}_0) - F_S(\boldsymbol{w}^{*})\right] + \frac{LG^2\left(\rho^2L + \eta\right)}{4\mu\rho\left(\mu + L\right)}.
		\end{aligned}
	\end{equation*}
\end{theorem}
\begin{proof}
	From Taylor's theorem, there exists a $\widehat{\boldsymbol{w}}_t$ such that
	\begin{equation*}
		\begin{aligned}
			F_S(\boldsymbol{w}_{t+1}) & = F_S\left(\boldsymbol{w}_t - \eta \nabla f(\boldsymbol{w}_{t}^{asc})\right) \\
			& = F_S(\boldsymbol{w}_t) - \eta \left<\nabla f(\boldsymbol{w}_{t}^{asc}), \nabla F_S(\boldsymbol{w}_{t}) \right> + \frac{\eta^2}{2}\nabla f(\boldsymbol{w}_{t}^{asc})^T \nabla^2 F_S(\widehat{\boldsymbol{w}}_t)\nabla f(\boldsymbol{w}_{t}^{asc}) \\
			& \leq F_S(\boldsymbol{w}_t) - \eta \left<\nabla f(\boldsymbol{w}_{t}^{asc}), \nabla F_S(\boldsymbol{w}_{t}) \right> + \frac{\eta^2 L}{2}  \|\nabla f(\boldsymbol{w}_{t}^{asc}) \|_2^2 \\
			& \leq F_S(\boldsymbol{w}_t) - \eta \left<\nabla f(\boldsymbol{w}_{t}^{asc}), \nabla F_S(\boldsymbol{w}_{t}) \right>+ \frac{\eta^2 LG^2}{2}.
		\end{aligned}
	\end{equation*}
	According to Lemma \ref{lemma: sam descent alignment}, it follows that
	\begin{equation*}
		\begin{aligned}
			\mathbb{E} \left<\nabla f(\boldsymbol{w}_t^{asc}), \nabla F_S(\boldsymbol{w}_t)\right> 
			& \geq \rho(\mu + L)\left\|\nabla F_S(\boldsymbol{w}_t)\right\|_2^2 - \frac{\rho^2L^2G^2}{2} \\
			&\geq 2\mu\rho(\mu + L)\left[F_S(\boldsymbol{w}_t) - F_S(\boldsymbol{w}^{*})\right] - \frac{\rho^2L^2G^2}{2},
		\end{aligned}
	\end{equation*}
	where the last inequality is due to Polyak-\L{}ojasiewicz condition as a result of being $\mu$-strongly convex.
	Subtracting $F_S(\boldsymbol{w}^{*})$ from both sides and taking expectations, we obtain
	\begin{equation*}
		\begin{aligned}
			\mathbb{E} \left[F_S(\boldsymbol{w}_{t+1})-F_S(\boldsymbol{w}^{*})\right] \leq \left[1 - 2\eta\mu\rho(\mu + L)\right] \mathbb{E}\left[F_S(\boldsymbol{w}_t) - F_S(\boldsymbol{w}^{*})\right] + \frac{\eta\rho^2L^2G^2}{2}+\frac{\eta^2G^2L}{2}.
		\end{aligned}
	\end{equation*}
	Recursively applying the above inequality and summing up the geometric series yields
	\begin{equation*}
		\begin{aligned}
			\mathbb{E} \left[F_S(\boldsymbol{w}_T)-F_S(\boldsymbol{w}^{*})\right] 
			\leq \left[1 - 2\eta\mu\rho(\mu + L)\right]^T \mathbb{E}\left[F_S(\boldsymbol{w}_0) - F_S(\boldsymbol{w}^{*})\right] + \frac{LG^2\left(\rho^2L + \eta\right)}{4\mu\rho\left(\mu + L\right)},
		\end{aligned}
	\end{equation*}
	thus concluding the proof.
\end{proof}
\begin{remark}
	Under a similar argument, we can establish that $\varepsilon_{\mathrm{opt}}^{\mathrm{sgd}}$ is bounded by $\frac{\eta L G^2}{4\mu}$ that vanishes when the learning rate $\eta$ becomes infinitesimally small.
	By contrast, the upper bound of $\varepsilon_{\mathrm{opt}}^{\mathrm{sam}}$ consists of a constant $\frac{\rho L^2 G^2}{4\mu(\mu+L)}$, implying that SAM will never converge to the minimum unless $\rho$ decays to zero as well.
	While we often use a fixed $\rho$ in practice to train neural networks, this observation highlights that $\rho$ should also be adjusted according to the learning rate to achieve a lower optimization error. 
	We note that while SAM consistently achieves a tighter upper bound over the generalization error than SGD, this theorem suggests that it does not necessarily perform better on unseen data because $\varepsilon_{\mathrm{opt}}^{\mathrm{sam}}$ is not always smaller than $\varepsilon_{\mathrm{opt}}^{\mathrm{sgd}}$. 
	Therefore, it requires particular attention in hyper-parameter tuning to promote the generalization performance.
	Moreover, if $\eta$ dominates over $\rho^2L$, this theorem suggests that the optimization error will decrease with $\rho$.
	On the contrary, if $\rho^2L \gg \eta$, the optimization error will increase with $\rho$.
\end{remark}
\indent
Combining the previous results, we are able to present an upper bound over the expected excess risk of the SAM algorithm.
\begin{theorem}
	\label{theorem: sam excess risk}
	Under assumptions and parameter settings in Theorems \ref{theorem: sam stability} and \ref{theorem: sam convergence}, the expected excess risk $\varepsilon_{\mathrm{exc}}^{\mathrm{sam}}$ of the output $\boldsymbol{w}_T$ obeys $\varepsilon_{\mathrm{exc}}^{\mathrm{sam}}\leq \varepsilon_{\mathrm{gen}}^{\mathrm{sam}} + \varepsilon_{\mathrm{opt}}^{\mathrm{sam}}$, where $\varepsilon_{\mathrm{gen}}^{\mathrm{sam}}$ and $\varepsilon_{\mathrm{opt}}^{\mathrm{sam}}$ are given by Theorems \ref{theorem: sam stability} and \ref{theorem: sam convergence}, respectively. Furthermore, as $T$ grows to infinity, we have
	\begin{equation*}
		\varepsilon_{\mathrm{exc}}^{\mathrm{sam}}\leq \frac{2G^2(\mu + L)}{n \mu L(1+\mu\rho)} + \frac{LG^2\left(\rho^2L + \eta\right)}{4\mu\rho\left(\mu + L\right)}.
	\end{equation*}
\end{theorem}
\begin{proof}
	This result is a direct consequence of $T\to\infty$.
\end{proof}
\begin{figure}[t]
	\centering
	\begin{subfigure}[b]{0.49\textwidth}
		\centering
		\includegraphics[width=\textwidth, clip, trim= 0 0 0 0]{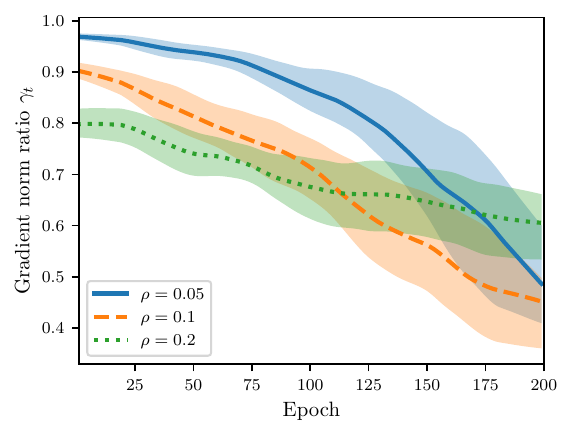}
		\caption{ResNet-20 on CIFAR-10}
	\end{subfigure}
	\begin{subfigure}[b]{0.49\textwidth}
		\centering
		\includegraphics[width=\textwidth, clip, trim= 0 0 0 0]{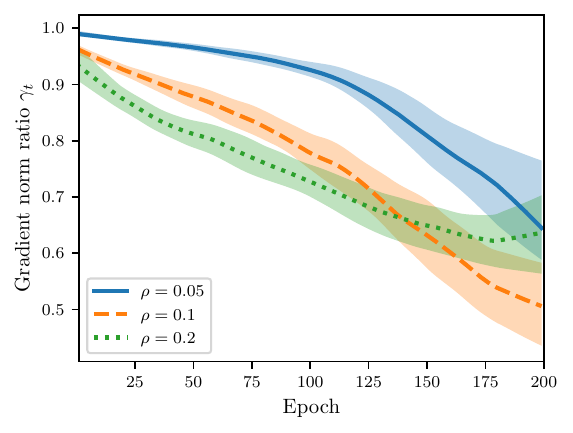}
		\caption{ResNet-56 on CIFAR-100}
	\end{subfigure}
	\caption{Evolution of the ratio $\gamma_t$ of the gradient norm of the ascent step $\|\nabla F_{\Omega_t}(\boldsymbol{w}_t)\|_2$ to that of the descent step $\|\nabla F_{\Omega_t}(\boldsymbol{w}^{asc}_t)\|_2$ throughout training.
		Both neural networks are trained up to 200 epochs using the SAM optimizer with different perturbation radius $\rho\in \{0.01, 0.05, 0.2\}$.
	}
	\label{fig: gradient norm ratio}
\end{figure}
\subsection{Expected Excess Risk Analysis of SSAM}
\label{subsection: Expected Excess Risk Analysis of StableSAM}
Now we continue to investigate the stability of sharpness-aware training when the renormalization strategy is applied.
Compared to SAM,  we demonstrate that SSAM allows for a relatively larger learning rate without performance deterioration.
\subsubsection{Stability}
For a fixed perturbation radius $\rho$, as shown in Figure \ref{fig: gradient norm ratio}, the renormalization factor $\gamma_t$ tends to decrease throughout training and is smaller than $1$.
Therefore, we can impose another assumption as follows.
\begin{assumption}
	\label{assumption: bounded renormalization ratio}
	Suppose that there exist a constant $\gamma_\mathrm{upp}$ so that $\gamma_t$ is bounded for all $1\leq t\leq T$
	\begin{equation*}
		0 < \gamma_t \leq \gamma_\mathrm{upp} < 1.
	\end{equation*}
\end{assumption}
Notice that the constant $\gamma_\mathrm{upp}$ is not universal but problem-specific.
Under this assumption, we can derive a similar growth rate of $\|\boldsymbol{v}_t - \boldsymbol{w}_t \|_2$ as Lemma \ref{lemma: sam same example}.
\begin{lemma}
	\label{lemma: tight stablesam same example}
	Let Assumption \ref{assumption: bounded renormalization ratio} hold and assume that the per-example loss function $f(\boldsymbol{w}, z)$ is $\mu$-strongly convex, $L$-smooth and $G$-Lipschitz continuous with respect to the first argument $\boldsymbol{w}$. Suppose that at step $t$, the examples selected by SSAM are the same in $S$ and $S^\prime$ and the corresponding update rules are denoted by
	$\boldsymbol{w}_{t+1} = \boldsymbol{w}_t - \eta \gamma_t\nabla f(\boldsymbol{w}_t^{asc}, z)$ and $\boldsymbol{v}_{t+1} = \boldsymbol{v}_t - \eta \gamma_t\nabla f(\boldsymbol{v}_t^{asc}, z)$, respectively.
	Then, it follows that for all $1\leq t\leq T$
	\begin{equation}
		\label{eq: ssam same example}
		\|\boldsymbol{v}_{t+1} - \boldsymbol{w}_{t+1} \|_2 \leq \left(1 -  \left(1+\mu\rho\right) \frac{\gamma_t\eta\mu L}{\mu+L}\right)  \|\boldsymbol{v}_{t} - \boldsymbol{w}_{t} \|_2,
	\end{equation}
	where the learning rate $\eta$ satisfies that
	\begin{equation}
		\label{eq: ssam learning rate constraint}
		\eta \leq \frac{1}{\gamma_\mathrm{upp}}\left[\frac{2}{\mu + L} - \frac{\mu + L}{2\mu L (\mu/\rho L^2 + 1)}\right].
	\end{equation}
\end{lemma}
\begin{proof}
	The proof is similar to Lemma \ref{lemma: sam same example}.
	According to the update rule of SSAM, we have
	\begin{align*}
		& \|\boldsymbol{v}_{t+1} - \boldsymbol{w}_{t+1} \|_2^2 
		={}  \| \boldsymbol{v}_t - \gamma_t\eta\nabla f(\boldsymbol{v}_t^{asc}) -\left(\boldsymbol{w}_t - \gamma_t\eta\nabla f(\boldsymbol{w}_t^{asc})\right)  \|_2^2 \\
		&{}={}  \|\boldsymbol{v}_{t} - \boldsymbol{w}_{t} \|_2^2 - 2\gamma_t\eta \left<\boldsymbol{v}_{t} - \boldsymbol{w}_{t}, \nabla f(\boldsymbol{v}_t^{asc})-\nabla f(\boldsymbol{w}_t^{asc}) \right> + \gamma_t^2\eta^2 \|\nabla f(\boldsymbol{v}_t^{asc})-\nabla f(\boldsymbol{w}_t^{asc}) \|_2^2 \\
		& =\begin{aligned}[t]
			& \|\boldsymbol{v}_{t} - \boldsymbol{w}_{t} \|_2^2 - 2\gamma_t\eta \left<\boldsymbol{v}_{t}^{asc} - \boldsymbol{w}_{t}^{asc}, \nabla f(\boldsymbol{v}_t^{asc})-\nabla f(\boldsymbol{w}_t^{asc}) \right>\\ 
			&\quad+ 2\gamma_t\rho\eta \left<\nabla f(\boldsymbol{v}_{t}) - \nabla f(\boldsymbol{w}_{t}), \nabla f(\boldsymbol{v}_t^{asc})-\nabla f(\boldsymbol{w}_t^{asc}) \right> + \gamma_t^2\eta^2 \|\nabla f(\boldsymbol{v}_t^{asc})-\nabla f(\boldsymbol{w}_t^{asc}) \|_2^2 
		\end{aligned}\\
		&\stackrel{\textcircled{1}}{\leq} \begin{aligned}[t]
			&\left(1 - 2 \left(1+\mu \rho \right) \frac{\gamma_t\eta\mu L}{\mu+L}\right)  \|\boldsymbol{v}_{t} - \boldsymbol{w}_{t} \|_2^2 - 2\left(\frac{\mu\rho}{L^2} + \rho^2\right)\frac{\gamma_t\eta\mu L}{\mu+L} \|\nabla f(\boldsymbol{v}_t)-\nabla f(\boldsymbol{w}_t) \|_2^2\\
			&\quad+2\gamma_t\rho\eta \left<\nabla f(\boldsymbol{v}_{t}) - \nabla f(\boldsymbol{w}_{t}), \nabla f(\boldsymbol{v}_t^{asc})-\nabla f(\boldsymbol{w}_t^{asc}) \right> + \left(\gamma_t^2\eta^2- \frac{2\gamma_t\eta}{\mu + L}\right) \|\nabla f(\boldsymbol{v}_t^{asc})-\nabla f(\boldsymbol{w}_t^{asc}) \|_2^2 
		\end{aligned}\\
		&\stackrel{\textcircled{2}}{\leq} \begin{aligned}[t]
			&\left(1 - 2 \left(1+\mu \rho \right) \frac{\gamma_t\eta\mu L}{\mu+L}\right)  \|\boldsymbol{v}_{t} - \boldsymbol{w}_{t} \|_2^2 + \left[\frac{\rho^2\gamma_t\eta}{\frac{2}{\mu+L} - \gamma_t\eta} - 2\left(\frac{\mu\rho}{L^2} + \rho^2\right)\frac{\gamma_t\eta\mu L}{\mu+L}\right] \|\nabla f(\boldsymbol{v}_t)-\nabla f(\boldsymbol{w}_t) \|_2^2\\
			&\quad+ \left(\gamma_t^2\eta^2- \frac{2\gamma_t\eta}{\mu + L}\right)\left[\left(\nabla f\left(\boldsymbol{v}_t^{asc}\right)-\nabla f\left(\boldsymbol{w}_t^{asc}\right)\right) - \frac{\rho}{\frac{2}{\mu+L} - \gamma_t\eta}\left(\nabla f\left(\boldsymbol{v}_{t}\right) - \nabla f\left(\boldsymbol{w}_{t}\right)\right) \right]^2
		\end{aligned}\\
		&\stackrel{\textcircled{3}}{\leq}{}\left(1 - 2 \left(1+\mu\rho\right) \frac{\gamma_t\eta\mu L}{\mu+L}\right)  \|\boldsymbol{v}_{t} - \boldsymbol{w}_{t} \|_2^2,
	\end{align*}
	where $\textcircled{1}$ is due to the coercivity of the loss function (cf. Appendix \ref{appendix: coeveritty of strongly convex})  that
	\begin{equation*}
		\left<\nabla f(\boldsymbol{v}_t^{asc})-\nabla f(\boldsymbol{w}_t^{asc}), \boldsymbol{v}_{t}^{asc} - \boldsymbol{w}_t^{asc} \right> \geq \frac{\mu L}{\mu + L}  \|\boldsymbol{v}_{t}^{asc} - \boldsymbol{w}_t^{asc} \|_2^2 + \frac{1}{\mu + L} \|\nabla f(\boldsymbol{v}_t^{asc})-\nabla f(\boldsymbol{w}_t^{asc}) \|_2^2.
	\end{equation*}
	Moreover, \textcircled{3} holds since the last two terms of \textcircled{2} are smaller than zero provided that the learning rate $\eta$ satisfies the given condition.
	Consequently, we have
	\begin{equation*}
		\|\boldsymbol{v}_{t+1} - \boldsymbol{w}_{t+1} \|_2 \leq \left(1 - 2 \left(1+\mu\rho\right) \frac{\gamma_t\eta\mu L}{\mu+L}\right)^{1/2}  \|\boldsymbol{v}_{t} - \boldsymbol{w}_{t} \|_2\leq \left(1 -  \left(1+\mu\rho\right) \frac{\gamma_t\eta\mu L}{\mu+L}\right)  \|\boldsymbol{v}_{t} - \boldsymbol{w}_{t} \|_2,
	\end{equation*}
	where the last inequality is due to the fact that $\sqrt{1-x}\leq 1-x/2$ holds for all $x\in[0, 1]$.
\end{proof}
On the other hand, when the examples selected from $S$ and $S^\prime$ are different, we can obtain a similar result as Lemma \ref{lemma: sam different example}.
\begin{lemma}
	\label{lemma: tight stablesam different example}
	Assume the same settings as in Lemma \ref{lemma: tight stablesam same example}. For the $t$-th iteration, suppose that the examples selected by SSAM are different in $S$ and $S^\prime$ and that $\boldsymbol{w}_{t+1} = \boldsymbol{w}_t - \eta\gamma_t \nabla f(\boldsymbol{w}_t^{asc}, z)$ and $\boldsymbol{v}_{t+1} = \boldsymbol{v}_t - \eta\gamma_t \nabla f(\boldsymbol{v}_t^{asc}, z^\prime)$. Consequently, we obtain
	\begin{equation*}
		\|\boldsymbol{v}_{t+1} - \boldsymbol{w}_{t+1} \|_2 \leq \left(1 -  \left(1+\mu\rho\right) \frac{\gamma_t\eta\mu L}{\mu+L}\right) \|\boldsymbol{v}_t - \boldsymbol{w}_t \|_2 + 2\eta \gamma_t G.
	\end{equation*}
\end{lemma}
\begin{proof}
	The proof is straightforward. 
	It follows immediately from
	\begin{equation*}
		\begin{aligned}
			\|\boldsymbol{v}_{t+1} - \boldsymbol{w}_{t+1} \|_2  & =  \|\boldsymbol{v}_t - \eta\gamma_t\nabla f(\boldsymbol{v}_t^{asc}, z^\prime) - \left(\boldsymbol{w}_t - \eta\gamma_t\nabla f(\boldsymbol{v}_t^{asc}, z^\prime)\right) - \eta\gamma_t\left(\nabla f(\boldsymbol{v}_t^{asc}, z^\prime) - \nabla f(\boldsymbol{w}_t^{asc}, z)\right) \|_2  \\
			& \leq  \|\boldsymbol{v}_t - \eta\gamma_t\nabla f(\boldsymbol{v}_t^{asc}, z^\prime) - \left(\boldsymbol{w}_t - \eta\gamma_t\nabla f(\boldsymbol{v}_t^{asc}, z^\prime)\right)  \|_2 + \eta\gamma_t  \|\nabla f(\boldsymbol{v}_t^{asc}, z^\prime) - \nabla f(\boldsymbol{w}_t^{asc}, z) \|_2 \\
			&\leq \left(1 -  \left(1+\mu\rho\right) \frac{\gamma_t\eta\mu L}{\mu+L}\right)  \|\boldsymbol{v}_{t} - \boldsymbol{w}_{t} \|_2 + 2\eta\gamma_t G,
		\end{aligned}
	\end{equation*}
	thus concluding the proof.
\end{proof}
With the above two lemmas, we can show that SSAM consistently performs better than SAM in terms of the generalization error.
Before that, we need to introduce an auxiliary lemma as follows.
\begin{lemma}
	\label{lemma: recursive inequality}
	Consider a sequence $\gamma_1, \ldots, \gamma_T$, where $0 < \gamma_k < 1$ for any $1\leq k \leq T$.
	Denote the maximum of the first $k$ elements by $\gamma_{max}^k$. Then, for any constants $\alpha>0$, $0 < \beta < 1/\gamma_{max}^T$, and $i=1, 2, \ldots$, the following inequality holds
	\begin{equation*}
		\label{eq: gamma square relation}
		\left(1 - \gamma_{k+1} \beta \right) \Psi(k) + \left(\gamma_{k+1}\right)^i\alpha \leq \Psi(k+1),
	\end{equation*}
	where
	\begin{equation*}
		\Psi(k) = \left[\left(1 - \gamma_{max}^k\beta\right)^{k-1} + \ldots + \left(1 - \gamma_{max}^k\beta\right) + 1\right] \left(\gamma_{max}^k\right)^i\alpha.
	\end{equation*}
\end{lemma}
\begin{proof}
	To prove this result, we only need to substitute $\Psi(k)$ in. In the case of $\gamma_{k+1} \geq \gamma_{max}^k$, we have
	\begin{equation*}
		\begin{aligned}
			& \Psi(k+1) - \left(1 - \gamma_{k+1} \beta \right) \Psi(k) - \left(\gamma_{k+1}\right)^i\alpha \\
			& = \frac{\alpha}{\beta}\left(1 - \gamma_{k+1} \beta \right)\left\{\left(\gamma_{k+1}\right)^{i-1}\left[1 - \left(1 - \gamma_{k+1} \beta \right)^k\right] - \left(\gamma_{max}^k\right)^{i-1}\left[1 - \left(1 - \gamma_{max}^k \beta \right)^k\right]\right\} \geq 0.
		\end{aligned}
	\end{equation*}
	In the case of $\gamma_{k+1} \leq \gamma_{max}^k$, we also have
	\begin{equation*}
		\begin{aligned}
			\Psi(k+1) - \left(1 - \gamma_{k+1} \beta \right) \Psi(k) - \left(\gamma_{k+1}\right)^i\alpha \geq \alpha \gamma_{k+1}\left[\left(\gamma_{max}^k\right)^{i-1} - \left(\gamma_{k+1}\right)^{i-1}\right] \geq 0,
		\end{aligned}
	\end{equation*}
	thus concluding the proof.
\end{proof}
\begin{theorem}
	\label{theorem: stablesam tight generalization error}
	Under assumptions and parameter settings in Lemmas \ref{lemma: tight stablesam same example} and \ref{lemma: tight stablesam different example}.
	Suppose we run the SSAM iteration with constant learning rate $\eta$ satisfying \eqref{eq: ssam learning rate constraint} for $T$ steps.
	Then, SSAM satisfies uniform stability with
	\begin{equation*}
		\varepsilon_{\mathrm{gen}}^\mathrm{ssam} \leq \frac{2G^2(\mu + L)}{n \mu L(1+\mu\rho)}\left\{1 -\left[1- \left(1+\mu\rho\right)\frac{\gamma_\mathrm{upp}\eta \mu L}{\mu + L}\right]^T\right\}.
	\end{equation*}
\end{theorem}
\begin{proof}
	Define $\delta_t=  \|\boldsymbol{w}_t - \boldsymbol{v}_t \|_2$ to denote the Euclidean distance between $\boldsymbol{w}_t$ and $\boldsymbol{v}_t$ as training continues.
	Observe that at any step $t\leq T$, with a probability $1-1/n$, the selected examples from $S$ and $S^\prime$ are the same.
	In contrast, with a probability of $1/n$, the selected examples are different. This is because $S$ and $S^\prime$ only differ by one example. Therefore, from Lemmas \ref{lemma: tight stablesam same example} and \ref{lemma: tight stablesam different example}, we conclude that
	\begin{equation*}
		\begin{aligned}
			\mathbb{E}[\delta_t] &\leq \left(1-\frac{1}{n}\right)\left(1 -  \left(1+\mu\rho\right) \frac{\gamma_t\eta\mu L}{\mu+L}\right)\mathbb{E}[\delta_{t-1}] + \frac{1}{n}\left(1 -  \left(1+\mu\rho\right) \frac{\gamma_t\eta\mu L}{\mu+L}\right)\mathbb{E}[\delta_{t-1}] + \frac{2\eta \gamma_tG}{n} \\
			&= \left(1 -  \left(1+\mu\rho\right) \frac{\gamma_t\eta\mu L}{\mu+L}\right)\mathbb{E}[\delta_{t-1}] + \frac{2\eta \gamma_tG}{n}.
		\end{aligned}
	\end{equation*}
	Write $\beta = \eta\mu L\left(1+\mu\rho\right)/\left(\mu + L\right)$ and $\alpha=2\eta G/n$, we then unravel the above recursion and obtain from Lemma \ref{lemma: recursive inequality} that 
	\begin{equation*}
		\mathbb{E}[\delta_T] \leq \gamma_\mathrm{upp} \alpha\sum_{t=0}^{T-1}\left(1 - \gamma_\mathrm{upp}\beta\right)^t = \frac{2G(\mu + L)}{n \mu L(1+\mu\rho)}\left\{1 -\left[1- \left(1+\mu\rho\right)\frac{\gamma_\mathrm{upp}\eta \mu L}{\mu + L}\right]^T\right\}.
	\end{equation*}
	Plugging this inequality into (\ref{eq: stability lipschitz inequality}), we complete the proof.
\end{proof}
\begin{remark}
	Compared to SAM, this theorem indicates that the bound over generalization error can be further reduced by SSAM because the extra term $\gamma_\mathrm{upp}$ is smaller than $1$.
\end{remark}
\subsubsection{Convergence}
Similar to Theorem \ref{theorem: sam convergence}, we show that SSAM also converges to a noisy ball when the learning rate $\eta$ is fixed.
\begin{theorem}
	\label{theorem: stablesam convergence}
	Let Assumption \ref{assumption: bounded renormalization ratio} hold and suppose the loss function $f(\boldsymbol{w}, z)$ is $\mu$-strongly convex, $L$-smooth, and $G$-Lipschitz continuous with respect to the first argument $\boldsymbol{w}$. 
	Consider the sequence $\boldsymbol{w}_0, \boldsymbol{w}_1, \cdots, \boldsymbol{w}_T$ generated by running SSAM with a constant learning rate $\eta$ for $T$ steps. 
	Let $\boldsymbol{w}^* \in \arg \inf_{\boldsymbol{w}} F_S(\boldsymbol{w})$, it follows that
	\begin{equation*}
		\begin{aligned}
			\varepsilon_{\mathrm{opt}}^{\mathrm{ssam}} = \mathbb{E} \left[F_S(\boldsymbol{w}_T)-F_S(\boldsymbol{w}^{*})\right] 
			\leq \left[1 - \gamma_\mathrm{upp}\eta\mu\rho(\mu + L)\right]^T \mathbb{E}\left[F_S(\boldsymbol{w}_0) - F_S(\boldsymbol{w}^{*})\right] + \frac{LG^2\left(\rho^2L + \gamma_\mathrm{upp}\eta\right)}{4\mu\rho(\mu + L)}.
		\end{aligned}
	\end{equation*}
\end{theorem}
\begin{proof}
	The proof follows the same steps as Theorem \ref{theorem: sam excess risk}.
	From Taylor's theorem, there exists a $\widehat{\boldsymbol{w}}_t$ such that
	\begin{equation*}
		\begin{aligned}
			F_S(\boldsymbol{w}_{t+1}) & = F_S\left(\boldsymbol{w}_t - \gamma_t\eta \nabla f(\boldsymbol{w}_{t}^{asc})\right) \\
			& = F_S(\boldsymbol{w}_t) - \gamma_t\eta \left<\nabla f(\boldsymbol{w}_{t}^{asc}), \nabla F_S(\boldsymbol{w}_{t}) \right> + \frac{\gamma_t^2\eta^2}{2}\nabla f(\boldsymbol{w}_{t}^{asc})^T \nabla^2 F_S(\widehat{\boldsymbol{w}}_t)\nabla f(\boldsymbol{w}_{t}^{asc}) \\
			& \leq F_S(\boldsymbol{w}_t) - \gamma_t\eta \left<\nabla f(\boldsymbol{w}_{t}^{asc}), \nabla F_S(\boldsymbol{w}_{t}) \right> + \frac{\gamma_t^2\eta^2 L}{2}  \|\nabla f(\boldsymbol{w}_{t}^{asc}) \|_2^2 \\
			& \leq F_S(\boldsymbol{w}_t) - \gamma_t\eta \left<\nabla f(\boldsymbol{w}_{t}^{asc}), \nabla F_S(\boldsymbol{w}_{t}) \right>+ \frac{\gamma_t^2\eta^2 LG^2}{2}.
		\end{aligned}
	\end{equation*}
	According to Lemma \ref{lemma: sam descent alignment}, it follows that
	\begin{equation*}
		\begin{aligned}
			\mathbb{E} \left<\nabla f(\boldsymbol{w}_t^{asc}), \nabla F_S(\boldsymbol{w}_t)\right> 
			& \geq \rho(\mu + L)\left\|\nabla F_S(\boldsymbol{w}_t)\right\|_2^2 - \frac{\rho^2L^2G^2}{2} \\
			&\geq 2\mu\rho(\mu + L)\left[F_S(\boldsymbol{w}_t) - F_S(\boldsymbol{w}^{*})\right] - \frac{\rho^2L^2G^2}{2},
		\end{aligned}
	\end{equation*}
	where the last inequality is due to Polyak-\L{}ojasiewicz condition as a result of being $\mu$-strongly convex.
	Subtracting $F_S(\boldsymbol{w}^{*})$ from both sides and taking expectations, we obtain
	\begin{equation*}
		\begin{aligned}
			\mathbb{E} \left[F_S(\boldsymbol{w}_{t+1})-F_S(\boldsymbol{w}^{*})\right] \leq \left[1 - \gamma_t\eta\mu\rho(\mu + L)\right] \mathbb{E}\left[F_S(\boldsymbol{w}_t) - F_S(\boldsymbol{w}^{*})\right] + \frac{\gamma_t\eta\rho^2L^2G^2}{2}+\frac{\gamma_t^2\eta^2G^2L}{2}.
		\end{aligned}
	\end{equation*}
	Recursively applying Lemma \ref{lemma: recursive inequality} and summing up the geometric series yields
	\begin{equation*}
		\begin{aligned}
			\mathbb{E} \left[F_S(\boldsymbol{w}_T)-F_S(\boldsymbol{w}^{*})\right] 
			\leq \left[1 - \gamma_\mathrm{upp}\eta\mu\rho(\mu + L)\right]^T \mathbb{E}\left[F_S(\boldsymbol{w}_0) - F_S(\boldsymbol{w}^{*})\right] + \frac{LG^2\left(\rho^2L + \gamma_\mathrm{upp}\eta\right)}{4\mu\rho(\mu + L)},
		\end{aligned}
	\end{equation*}
	thus concluding the proof.
\end{proof}
\begin{remark}
	Since we require that $\gamma_\mathrm{upp}$ is smaller than $1$, compared to SAM, this theorem suggests that  SSAM nevertheless slows down the training process. 
\end{remark}
Combining these results, we can present an upper bound over the expected excess risk of the SSAM algorithm as follows.
\begin{theorem}
	\label{theorem: stablesam tight excess risk}
	Under assumptions and parameter settings in Theorems \ref{theorem: stablesam tight generalization error} and \ref{theorem: stablesam convergence}, the expected excess risk $\varepsilon_\mathrm{exc}^\mathrm{ssam}$ of the output $\boldsymbol{w}_T$ obeys $\varepsilon_\mathrm{exc}^\mathrm{ssam}\leq \varepsilon_\mathrm{opt}^\mathrm{ssam} + \varepsilon_\mathrm{gen}^\mathrm{ssam}$, where $\varepsilon_\mathrm{opt}^\mathrm{ssam}$ and $\varepsilon_\mathrm{gen}^\mathrm{ssam}$ are given by Theorems \ref{theorem: stablesam tight generalization error} and \ref{theorem: stablesam convergence}, respectively. Furthermore, as $T$ tends to infinity, we have
	\begin{equation*}
		\varepsilon_\mathrm{exc}^\mathrm{ssam}\leq \frac{2G^2(\mu + L)}{n\mu L (1+\mu{\rho})} + \frac{LG^2\left(\rho^2L + \gamma_\mathrm{upp}\eta\right)}{4\mu\rho(\mu + L)}.
	\end{equation*}
\end{theorem}
\begin{proof}
	This result follows immediately as $T\to\infty$.
\end{proof}
\begin{remark}
	This theorem implies that SSAM would eventually achieve a tighter bound over the expected excess risk than SAM when the model is trained for a sufficiently long time.
\end{remark}

\section{Experiments}
\label{sec:experiments}
In this section, we present the empirical results on a range of tasks.
From the perspective of algorithmic stability, we first investigate how SSAM ameliorates the issue of training instability with realistic data sets.
We then provide the convergence results on a quadratic loss function.
To demonstrate that the increased stability does not come at the cost of performance degradation, we also evaluate it on tasks such as training deep classifiers from scratch.
The results suggest that SSAM can achieve comparable or even superior performance compared to SAM.
For completeness, sometimes we also include the results of the standard formulation of SAM proposed by \citet{foret2020sharpness} and denote it by $\mathrm{SAM}^{\ast}$.
\begin{figure}[t]
	\centering
	\begin{subfigure}[b]{0.48\textwidth}
		\centering
		\includegraphics[width=\textwidth, clip, trim= 0 0 0 0]{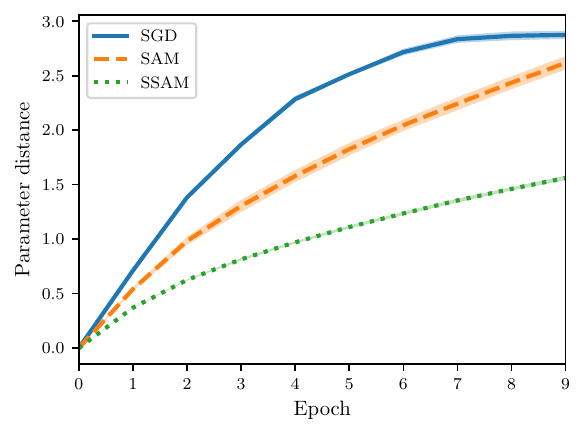}
		\caption{}
	\end{subfigure}
	\begin{subfigure}[b]{0.48\textwidth}
		\centering
		\includegraphics[width=\textwidth, clip, trim= 0 0 0 0]{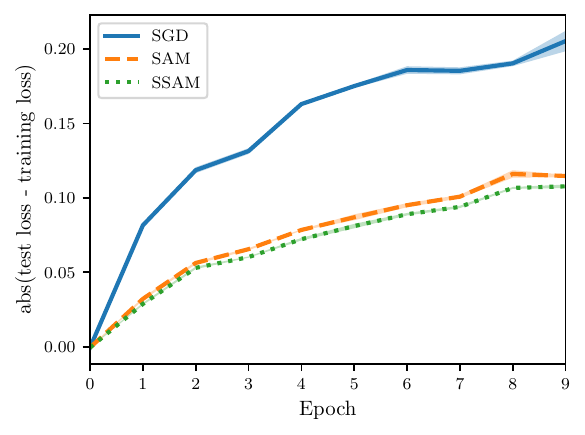}
		\caption{}
	\end{subfigure}
	\caption{Evolution of (a) parameter distance and (b) generalization gap as a function of epoch.
		The base model is a fully connected neural network and the data set is MNIST.
		All models are trained with a constant learning rate and neither momentum nor weight decay is employed.}
	\label{fig: Evolution of stability fcn mnist}
\end{figure}
\begin{figure}[t]
	\centering
	\begin{subfigure}[b]{0.48\textwidth}
		\centering
		\includegraphics[width=\textwidth, clip, trim= 0 0 0 0]{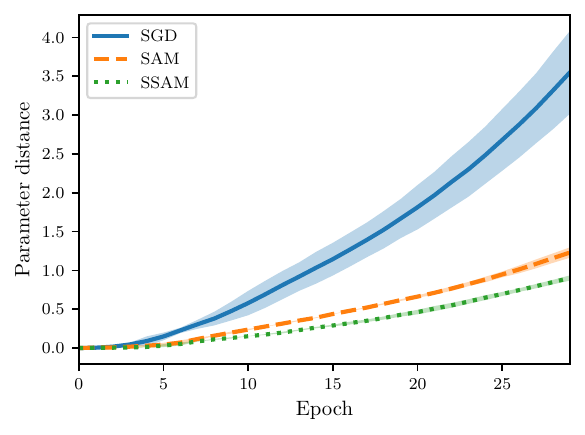}
		\caption{}
	\end{subfigure}
	\begin{subfigure}[b]{0.48\textwidth}
		\centering
		\includegraphics[width=\textwidth, clip, trim= 0 0 0 0]{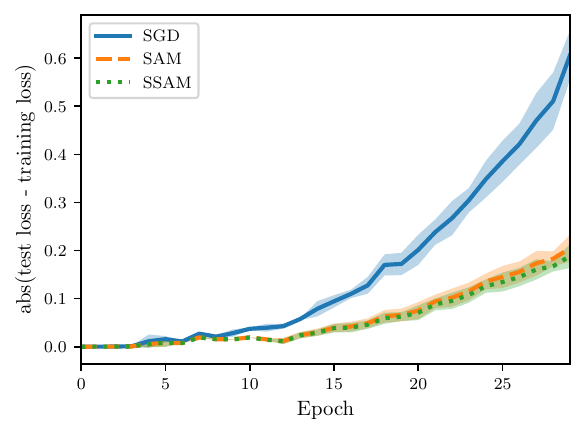}
		\caption{}
	\end{subfigure}
	\caption{Evolution of (a) parameter distance and (b) generalization gap as a function of epoch.
		The base model is LeNet and the data set is CIFAR-10.
		All models are trained with a constant learning rate and neither momentum nor weight decay is employed.}
	\label{fig: Evolution of stability lenet cifar10}
\end{figure}
\subsection{Algorithmic Stability}
In Section \ref{subsection: Expected Excess Risk Analysis of StableSAM}, we showed that SSAM can consistently perform better than SAM in terms of generalization error (see Theorems \ref{theorem: sam stability} and \ref{theorem: stablesam tight generalization error} for a comparison).
To verify this claim empirically, we follow the experimental settings of \citet{hardt2016train} and consider two proxies to measure the algorithmic stability.
The first is the Euclidean distance between the parameters of two identical models, namely, with the same architecture and initialization.
The second proxy is the generalization error which measures the difference between the training error and the test error.
\par
To construct two training sets $S$ and $S^\prime$ that differ in only one example, we first randomly remove an example from the given training set, and the remaining examples naturally constitute one set $S$.
Then we can create another set $S^\prime$ by replacing a random example of $S$ with the one previously deleted.
We restrict our attention to the task of image classification and adopt two different neural architectures: a simple fully connected neural network (FCN) trained on MNIST, and a LeNet \citep{lecun1998gradient} trained on CIFAR-10.
The FCN model consists of two hidden layers of 500 neurons, each of which is followed by a ReLU activation function.
To make our experiments more controllable, we exclude all forms of regularization such as weight decay and dropout.
We use the vanilla SGD (namely, mini-batch size is $1$) without momentum acceleration as the default base optimizer and train each model with a constant learning rate.
Of course, we also fix the random seed at each epoch to ensure that the order of examples in two training sets remains the same.
Additionally, we do not use data augmentation so that the distribution shift between training data and test data is minimal.
Moreover, we record the Euclidean distance and the generalization error once per epoch.
\par
As shown in Figures \ref{fig: Evolution of stability fcn mnist} and \ref{fig: Evolution of stability lenet cifar10}, there is a close correspondence between the parameter distance and the generalization error.
These two quantities often move in tandem and are positively correlated.
Moreover, when starting from the same initialization, models trained by SGD quickly diverge, whereas models trained by SAM and SSAM change slowly.
By comparing the training curves, we can further observe that SSAM is significantly less sensitive than SAM when the training set is modified.
\subsection{Convergence Results}
To empirically validate the convergence results of SAM and SSAM, here we consider a quadratic loss function of dimension $d=20$, 
\begin{equation*}
	f(x) = \frac{1}{2}x^T ({AA^T}/{2d}+\delta \mathbb{I})x,
\end{equation*}
where $A\in\mathbb{R}^{d\times 2d}$ is a random matrix with elements being standard Gaussian noise and $\delta$ is a small positive coefficient to ensure that the loss function is strongly convex.
Starting from a point sampled according to $\mathcal{N}(0, \mathbb{I})$, we optimize the loss function for one million steps with a constant learning rate of $1.0e^{-3}$.
To introduce stochasticity, we also perturb the gradient at each step with random noise from $\mathcal{N}(0, 1.0e^{-4})$.
\begin{figure}[t]
	\centering
	\begin{subfigure}[b]{0.49\textwidth}
		\centering
		\includegraphics[width=\textwidth, clip, trim= 0 0 0 0]{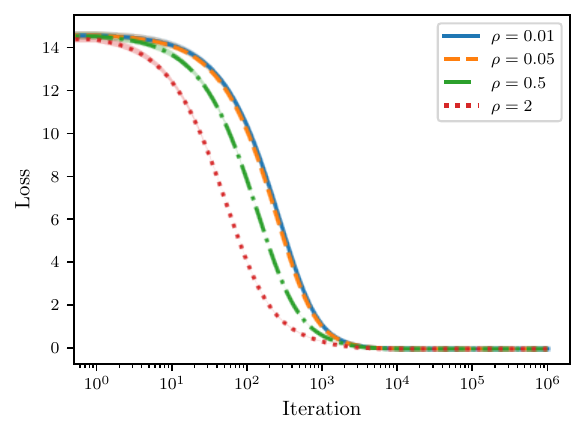}
		\caption{}
	\end{subfigure}
	\begin{subfigure}[b]{0.49\textwidth}
		\centering
		\includegraphics[width=\textwidth, clip, trim= 0 0 0 0]{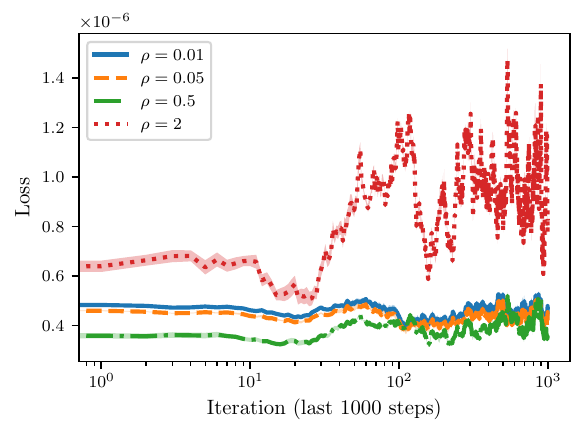}
		\caption{}
	\end{subfigure}
	\caption{The left panel illustrates the loss curves for SAM under different values of perturbation radius $\rho$ and the right panel displays the loss of the last 1000 steps.}
	\label{fig: sam convergence results}
\end{figure}
\begin{figure}[t]
	\centering
	\begin{subfigure}[b]{0.49\textwidth}
		\centering
		\includegraphics[width=\textwidth, clip, trim= 0 0 0 0]{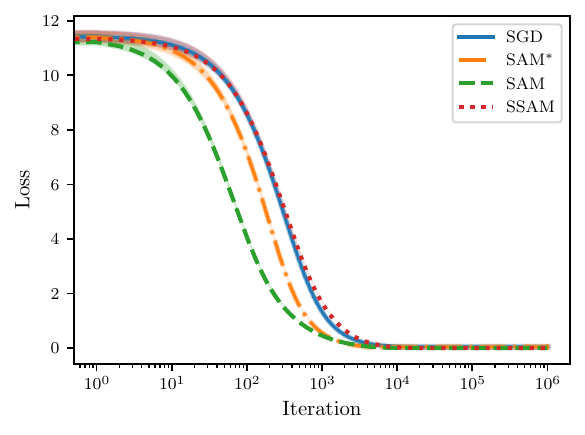}
		\caption{}
	\end{subfigure}
	\begin{subfigure}[b]{0.49\textwidth}
		\centering
		\includegraphics[width=\textwidth, clip, trim= 0 0 0 0]{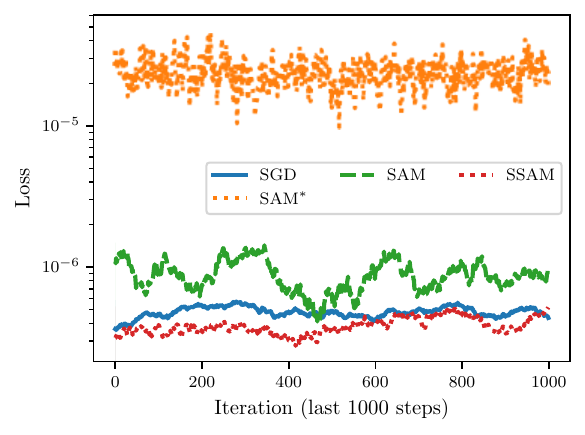}
		\caption{}
	\end{subfigure}
	\caption{The left panel illustrates the loss curves for different optimizers and the right panel displays the loss of the last 1000 steps.}
	\label{fig: opt comparison}
\end{figure}
\par
As depicted in Figure \ref{fig: sam convergence results}, we can observe that the convergence speed of SAM grows with the perturbation radius $\rho$.
More importantly, as we gradually increase $\rho$ from $0.01$ to $2$, the loss at the end of training first decreases and then starts to increase, suggesting that there indeed exists a tradeoff between $\rho^2L$ and the learning rate $\eta$ as predicted by Theorem \ref{theorem: sam convergence}.
We then compare SSAM against SAM and SGD in Figure \ref{fig: opt comparison} and find that SSAM indeed slows down the convergence speed.
But, just as implied by Theorem \ref{theorem: stablesam convergence}, it is able to achieve a lower loss than SAM when trained for a sufficiently long period.
Meanwhile, although SAM converges faster than SGD, it nevertheless converges to a larger noisy ball than SGD, which once again suggests that a careful choice of $\rho$ is critical to achieving a better generalization performance.
From Figure \ref{fig: opt comparison}(b), we can also observe that $\mathrm{SAM}^\ast$ seems to be more unstable than SAM because of the normalization step.
\subsection{Image Classification from Scratch}
\label{sec:Image Classification from Scratch}
We now continue to investigate how SSAM performs on real-world image classification problems.
The baselines include SGD, SAM \citep{andriushchenko2022towards},  $\mathrm{SAM}^\ast$ \citep{foret2020sharpness}, ASAM \citep{kwon2021asam}, and one-step GASAM \citep{zhang2022ga} that attempts to stabilize the training dynamics as well.
\begin{table}[t]
	\centering
	\caption{Results on CIFAR-10 and CIFAR-100. We run each model with three different random seeds and report the mean test accuracy (\%) along with the standard deviation. Text marked as bold indicates the best result.}
	\begin{adjustbox}{width=0.98\textwidth,center}
		\begin{tabular}{ccccccc}
			\toprule
			&       & ResNet-20        & ResNet-56        & ResNext-29-32x4d & WRN-28-10        & PyramidNet-110   \\ \midrule
			\multirow{7}{*}{CIFAR-10}  & SGD   & 92.78 $\pm$ 0.11 & 93.99 $\pm$ 0.19 & 95.47 $\pm$ 0.06 & 96.08 $\pm$ 0.16 & 96.02 $\pm$ 0.16 \\ \cmidrule{2-7}
			& $\mathrm{SAM}^\ast$   & 93.39 $\pm$ 0.14 & 94.93 $\pm$ 0.21 & 96.30 $\pm$ 0.01 & \textbf{96.91 $\pm$ 0.12} & 96.95 $\pm$ 0.06 \\ \cmidrule{2-7}
			& SAM   & 93.43 $\pm$ 0.24 & 94.92 $\pm$ 0.22 & 96.20 $\pm$ 0.08 & 96.55 $\pm$ 0.17 & 96.91 $\pm$ 0.16 \\ \cmidrule{2-7}
			& SSAM  & \textbf{93.46 $\pm$ 0.22} & \textbf{95.01 $\pm$ 0.19} & \textbf{96.33 $\pm$ 0.16} & 96.65 $\pm$ 0.18 & \textbf{97.04 $\pm$ 0.09} \\ \cmidrule{2-7}
			& ASAM  & 93.11 $\pm$ 0.23 & 94.51 $\pm$ 0.34 & 95.74 $\pm$ 0.06 & 96.24 $\pm$ 0.08 & 96.39 $\pm$ 0.14 \\ \cmidrule{2-7}
			& GASAM & 92.96 $\pm$ 0.14 & 94.18 $\pm$ 0.31 & 93.66 $\pm$ 0.92 & 95.75 $\pm$ 0.34 & 81.83 $\pm$ 1.58 \\ \midrule
			\multirow{7}{*}{CIFAR-100} & SGD   & 69.11 $\pm$ 0.11 & 72.38 $\pm$ 0.17 & 79.93 $\pm$ 0.15 & 80.42 $\pm$ 0.06 & 81.39 $\pm$ 0.31 \\ \cmidrule{2-7}
			& $\mathrm{SAM}^\ast$  & 70.30 $\pm$ 0.32 & 74.81 $\pm$ 0.07 & 81.09 $\pm$ 0.37 & \textbf{83.23 $\pm$ 0.19} & \textbf{84.03 $\pm$ 0.27} \\ \cmidrule{2-7}
			& SAM   & \textbf{70.77 $\pm$ 0.24} & 75.02 $\pm$ 0.19 & 81.25 $\pm$ 0.14 & 82.94 $\pm$ 0.35 & 83.68 $\pm$ 0.10 \\ \cmidrule{2-7}
			& SSAM  & 70.48 $\pm$ 0.18 & \textbf{75.11 $\pm$ 0.14} & \textbf{81.35 $\pm$ 0.13} & 82.80 $\pm$ 0.15 & 83.78 $\pm$ 0.17 \\ \cmidrule{2-7}
			& ASAM  & 69.57 $\pm$ 0.12 & 72.82 $\pm$ 0.32 & 80.01 $\pm$ 0.14 & 81.34 $\pm$ 0.31 & 82.04 $\pm$ 0.09 \\ \cmidrule{2-7}
			& GASAM & 69.02 $\pm$ 0.13 & 72.05 $\pm$ 1.09 & 77.81 $\pm$ 1.52 & 81.48 $\pm$ 0.31 & 45.59 $\pm$ 3.03 \\ \bottomrule
		\end{tabular}
	\end{adjustbox}
	\label{tab: test accuracy on cifar10-100}
\end{table}
\par
\textbf{CIFAR-10 and CIFAR-100.}
Here we adopt several popular backbones, ranging from basic ResNets \citep{he2016deep} to more advanced architectures such as WideResNet \citep{zagoruyko2016wide}, ResNeXt \citep{xie2017aggregated}, and PyramidNet \citep{han2017deep}.
To increase reproductivity, we decide to employ the standard implementations of these architectures that are encapsulated in a Pytorch package\footnote{Details can be found at \url{https://pypi.org/project/pytorchcv}.}.
Beyond the training and test set, we also construct a validation set containing 5000 images out of the training set.
Moreover, we only employ basic data augmentations such as horizontal flip, random crop, and normalization.
We set the mini-batch size to be 128 and each model is trained up to 200 epochs with a cosine learning rate decay \citep{loshchilov2016sgdr}.
The default base optimizer is SGD with a momentum of 0.9.
To determine the best choice of hyper-parameters for each backbone, slightly different from \citet{kwon2021asam,kim2022fisher}, we first use SGD to grid search the learning rate and the weight decay coefficient over \{0.01, 0.05, 0.1\} and \{1.0e-4, 5.0e-4, 1.0e-3\}, respectively.
For SAM and the variants, these two hyper-parameters are then fixed.
As suggested by \citet{kwon2021asam}, the perturbation radius $\rho$ of ASAM needs to be much larger, and we thus range it from \{0.5, 1.0, 2.0\}.
In contrast, we sweep the perturbation radius $\rho$ of other optimizers over \{0.05, 0.1, 0.2\}.
We run each model with three different random seeds and report the mean and the standard deviation of the accuracy on the test set.
\par
As shown in Table \ref{tab: test accuracy on cifar10-100}, apart from GASAM that even fails to converge for PyramidNet-110, both SAM and its variants are able to consistently perform better than the base optimizer SGD.
Meanwhile, it is worth noting that there is no significant difference between SAM \citep{andriushchenko2022towards} and $\mathrm{SAM}^\ast$ \citep{foret2020sharpness}, suggesting that the normalization term is not necessary for promoting generalization performance.
Focusing on the rows of SSAM and SAM, we further observe that SSAM can achieve a higher test accuracy than SAM on most backbones, though the improvements may not be significant.
\par
\textbf{ImageNet-1K \citep{deng2009imagenet}.} To investigate the performance of the renormalization strategy on a larger scale, we further evaluate it with the ImageNet-1K data set.
We only employ basic data augmentations, namely, resizing and cropping images to 224-pixel resolution and then normalizing them.
We adopt several typical architectures\footnote{Both models are trained with the timm library that is available at \url{https://github.com/huggingface/pytorch-image-models}.}, including two ResNets (ResNet-18/50), and two vision transformers (ViT-S-16/32) \citep{dosovitskiy2020image}.
ResNet-18 and ResNet-50 are trained for 90 and 100 epochs, respectively.
The default base optimizer is SGD with momentum acceleration, the peak learning rate is 0.1,  and the weight decay coefficient is 1.0e-4.
According to \citet{foret2020sharpness}, the perturbation radius $\rho$ is set to be 0.05.
For the vision transformer, the two models are trained up to 300 epochs and the default base optimizer is switched to AdamW. The peak learning rate is 3.0e-4 and the weight decay coefficient is 0.3.
The value of $\rho$ is 0.2 because the vision transformer favors larger $\rho$ than ResNet does \citep{chen2021vision}.
For both models, we use a constant mini-batch size of 256, and the cosine learning rate decay schedule is also employed.
As shown in Table \ref{tab: imagenet-acc}, the renormalization strategy remains effective on the ImageNet-1K data set. 
After applying the renormalization strategy to SAM, we can observe an improved top-1 accuracy on the validation set for all models, though the improvement is more pronounced for the two vision transformers.	
\begin{table}[t]
	\centering
	\caption{Top-1 accuracy (\%) on ImageNet-1K validation set with Inception-style data augmentation only. The base optimizer for ResNet is SGD with a momentum of 0.9. In contrast, the base optimizer for the vision transformer is AdamW.}
	\begin{adjustbox}{width=0.75\textwidth,center}
		\begin{tabular}{ccccc}
			\toprule
			& SGD/AdamW  & $\mathrm{SAM}^\ast$   & SAM &SSAM\\ \midrule
			ResNet-18  & 70.56 $\pm$ 0.03 & 70.74 $\pm$ 0.02  & 70.66 $\pm$ 0.12  & \textbf{70.76 $\pm$ 0.09} \\ \midrule
			ResNet-50 & 77.09 $\pm$ 0.12 & 77.81 $\pm$ 0.04  & 77.82 $\pm$ 0.08   & \textbf{77.89 $\pm$ 0.13} \\ \midrule
			ViT-S-32 & 65.42 $\pm$ 0.12 & 67.42 $\pm$ 0.21 & 69.98 $\pm$ 0.11    & \textbf{71.15 $\pm$ 0.18} \\ \midrule
			ViT-S-16   & 72.25 $\pm$ 0.09 & 73.81 $\pm$ 0.06  & 76.88 $\pm$ 0.25  & \textbf{77.41 $\pm$ 0.13}\\ \bottomrule
		\end{tabular}
	\end{adjustbox}
	\label{tab: imagenet-acc}
\end{table}
\subsection{Minima Analysis}
Finally, to gain a better understanding of SSAM, we further compare the differences in the sharpness of the minima found by different optimizers, which can be described by the dominant eigenvalue of the Hessian of the loss function \citep{foret2020sharpness, zhuang2021surrogate, kaddour2022flat}.
For this purpose, we train a ResNet-20 on CIFAR-10 and a ResNet-56 on CIFAR-100 using the same hyper-parameters and then estimate the top five eigenvalues of the Hessian.
\par
From Figure \ref{fig: sharpness analysis}, we can observe that compared to SGD, SAM significantly reduces the sharpness of the minima.
Meanwhile, it also can be observed that SSAM achieves the lowest eigenvalue, suggesting that the renormalization strategy is indeed beneficial in escaping saddle points and finding flatter regions of the loss landscape.
\begin{figure}[t]
	\centering
	\begin{subfigure}[b]{0.48\textwidth}
		\centering
		\includegraphics[width=\textwidth, clip, trim= 0 0 0 0]{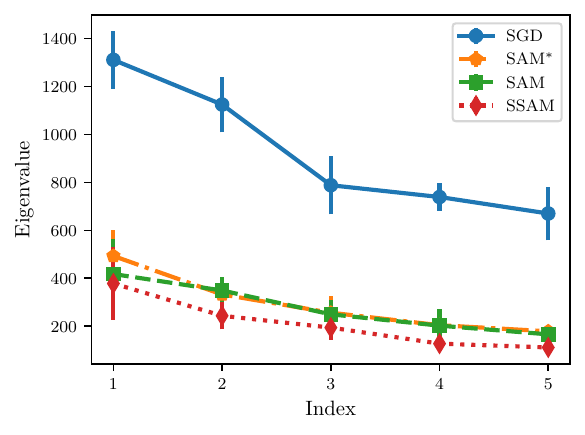}
		\caption{ResNet-20 on CIFAR-10}
	\end{subfigure}
	\begin{subfigure}[b]{0.48\textwidth}
		\centering
		\includegraphics[width=\textwidth, clip, trim= 0 0 0 0]{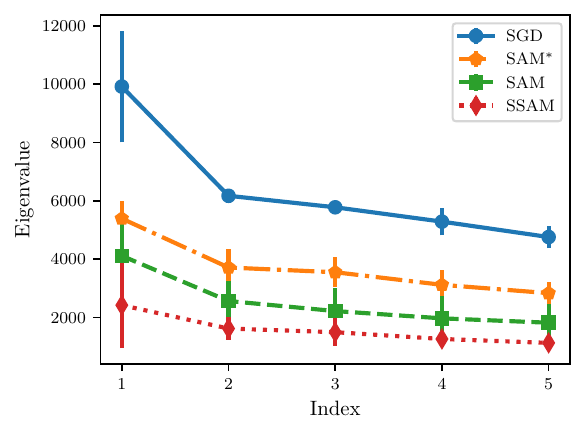}
		\caption{ResNet-56 on CIFAR-100}
	\end{subfigure}
	\caption{Illustration of the top five eigenvalues of the Hessian of the loss function, which is estimated using PyHessian \citep{yao2020pyhessian}.
		Since sharpness can be easily manipulated with the reparameterization trick \citep{dinh2017sharp}, following \citet{jiangfantastic2020}, we remove the batch normalization before computing the Hessian by fusing the normalization layer with the preceding convolution layer.
	}
	\label{fig: sharpness analysis}
\end{figure}
\section{Conclusion}
In this paper, we proposed a renormalization strategy to mitigate the issue of instability in sharpness-aware training.
We also evaluated its efficacy, both theoretically and empirically.
Following this line, we believe several directions deserve further investigation. Although we have verified that SSAM and SAM both can greatly improve the generalization performance over SGD, it remains unknown whether they converge to the same attractor of minima, properties of which might significantly differ from those found by SGD \citep{kaddour2022flat}.
Moreover, probing to what extent the renormalization strategy reshapes the optimization trajectory or the parameter space it explores is also of interest.
Another intriguing direction involves controlling the renormalization factor during the training process, for example, by imposing explicit constraints on its bounds or adjusting the perturbation radius according to the gradient norm of the ascent step.
Finally, the influence of renormalization strategy on adversarial robustness should also be investigated \citep{wei2023sharpness}.

\acks{This work is supported in part by the National Key Research and Development Program of China under Grant 2020AAA0105601, in part by the National Natural Science Foundation of China under Grants 12371512 and 62276208, and in part by the Natural Science Basic Research Program of Shaanxi Province 2024JC-JCQN-02.}

\appendix
\section{Training Instability on Realistic Neural Networks}
\label{sec: Training Instability on Realistic Neural Networks}
\begin{figure}[t]
	\centering
	\begin{subfigure}[b]{0.48\textwidth}
		\centering
		\includegraphics[width=\textwidth, clip, trim= 0 0 0 0]{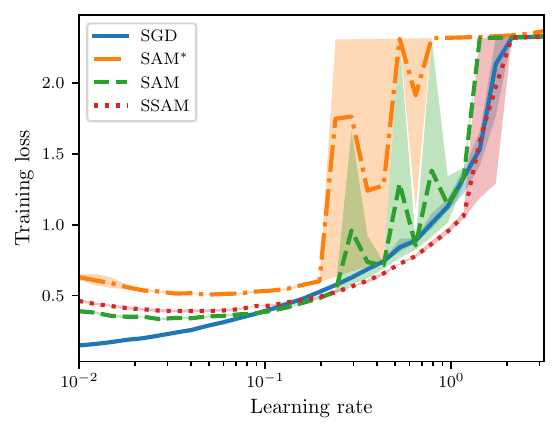}
		\caption{}
	\end{subfigure}
	\begin{subfigure}[b]{0.48\textwidth}
		\centering
		\includegraphics[width=\textwidth, clip, trim= 0 0 0 0]{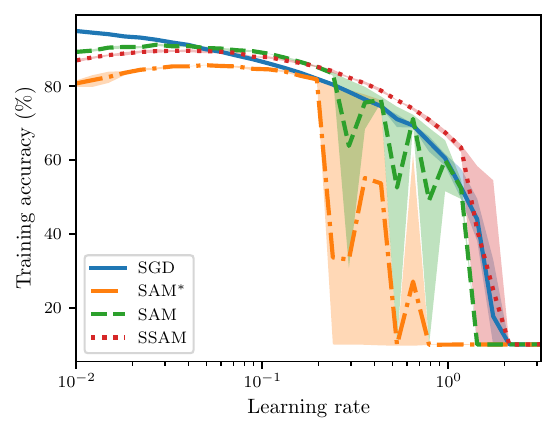}
		\caption{}
	\end{subfigure}
	\caption{Curves of (a) training loss and (b) training accuracy of different optimizers as a function of the learning rate. 
		Notice that both metrics are evaluated on the model of the last epoch. 
		The backbone is ResNet-20 and the data set is CIFAR-10.}
	\label{fig: success rate resnet20-cifar10}
\end{figure}
\begin{figure}[t]
	\centering
	\begin{subfigure}[b]{0.48\textwidth}
		\centering
		\includegraphics[width=\textwidth, clip, trim= 0 0 0 0]{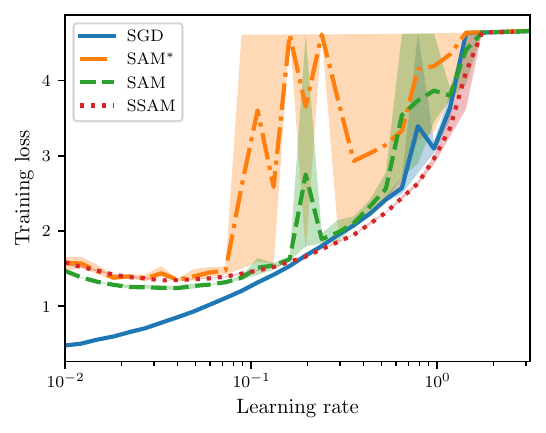}
		\caption{}
	\end{subfigure}
	\begin{subfigure}[b]{0.48\textwidth}
		\centering
		\includegraphics[width=\textwidth, clip, trim= 0 0 0 0]{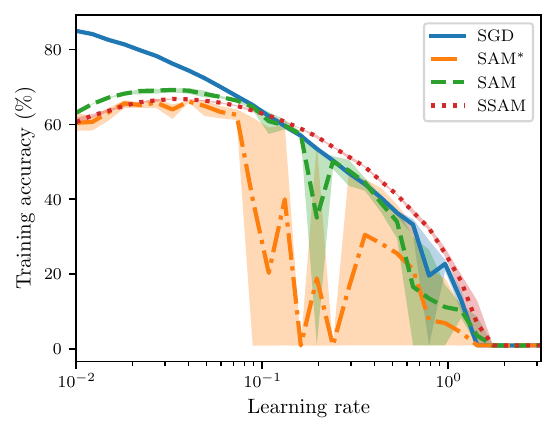}
		\caption{}
	\end{subfigure}
	\caption{Curves of (a) training loss and (b) training accuracy of different optimizers as a function of the learning rate. 
		Notice that both metrics are evaluated on the model of the last epoch. 
		The backbone is ResNet-56 and the data set is CIFAR-100.}
	\label{fig: success rate resnet56-cifar100}
\end{figure}
To examine the training stability on real-world applications, we also train a ResNet-20 on CIFAR-10 and a ResNet-56 on CIFAR-100 with different learning rates that are equispaced between 0.01 and 3.16 on the logarithm scale.
The default optimizer is SGD with a mini-batch size of $128$ and each model is trained up to $200$ epochs.
To make the difference more significant, we use a relatively large value of $\rho=1.0$, and the learning rate is not decayed throughout training.
\par
In Figures \ref{fig: success rate resnet20-cifar10} and \ref{fig: success rate resnet56-cifar100}, we report the metrics of loss and accuracy on the training set at the end of training. 
When the learning rate is small, we can observe that SGD attains the lowest loss and SAM performs better than SSAM.
As we continue to increase the learning rate, however, SAM becomes highly unstable and finally fails to converge.
As a comparison, we can observe that SSAM is more stable than SAM and even can achieve a lower loss and a higher accuracy than SGD in a relatively large range of learning rates.
Notice that SAM$^\ast$ is still much more unstable than SAM because of the normalization step.
\section{Coercivity of Strongly Convex Function}
\label{appendix: coeveritty of strongly convex}
\begin{lemma}
	\label{lemma: coeveritty of strongly convex}
	A function $f(\boldsymbol{w}): \mathbb{R}^d\mapsto\mathbb{R}_{+}$ is $\mu$-strongly convex and $L$-smooth, for all $\boldsymbol{w}$, $\boldsymbol{v}\in\mathbb{R}^d$, we have
	\begin{equation*}
		\left<\nabla f(\boldsymbol{v}) - \nabla f(\boldsymbol{w}), \boldsymbol{v} -\boldsymbol{w} \right> \geq \frac{\mu L}{\mu + L}  \|\boldsymbol{v} - \boldsymbol{w} \|_2^2 + \frac{1}{\mu + L} \|\nabla f(\boldsymbol{v}) - \nabla f(\boldsymbol{w}) \|_2^2.
	\end{equation*}
\end{lemma}
\begin{proof}
	Consider the function $\varphi(\boldsymbol{w})= f(\boldsymbol{w}) - \frac{\mu}{2}  \|\boldsymbol{w} \|_2^2$, which is convex with $(L-\mu)$-smooth by appealing to the fact that $f(\boldsymbol{w})$ is $\mu$-strongly convex and $L$-smooth. Therefore, it follows that 
	\begin{equation*}
		\left<\nabla\varphi(\boldsymbol{v}) - \nabla\varphi(\boldsymbol{w}), \boldsymbol{v} -\boldsymbol{w} \right> \geq \frac{1}{L-\mu} \|\nabla\varphi(\boldsymbol{v}) - \nabla\varphi(\boldsymbol{w}) \|_2^2.
	\end{equation*}
	On the other hand, 
	\begin{equation*}
		\left<\nabla\varphi(\boldsymbol{v}) - \nabla\varphi(\boldsymbol{w}), \boldsymbol{v} -\boldsymbol{w} \right> =  \left<\nabla f(\boldsymbol{v}) - \nabla f(\boldsymbol{w}), \boldsymbol{v} -\boldsymbol{w} \right> -\mu  \left<\boldsymbol{v} -\boldsymbol{w}, \boldsymbol{v} -\boldsymbol{w} \right>.
	\end{equation*}
	Substituting the preceding inequality in, we have 
	\begin{equation*}
		\begin{aligned}
			\left<\nabla f(\boldsymbol{v}) - \nabla f(\boldsymbol{w}), \boldsymbol{v} -\boldsymbol{w} \right> & \geq  \frac{1}{L - \mu} \| \nabla f(\boldsymbol{v}) - \nabla f(\boldsymbol{w})  - \mu (\boldsymbol{v} - \boldsymbol{w}) \|_2^2  + \mu \|\boldsymbol{v} - \boldsymbol{w}\|_2^2.
		\end{aligned}
	\end{equation*}
	Expanding the first term on the right side,  it follows that
	\begin{equation*}
		\left<\nabla f(\boldsymbol{v}) - \nabla f(\boldsymbol{w}), \boldsymbol{v} -\boldsymbol{w} \right> \geq \frac{\mu L}{\mu + L}  \|\boldsymbol{v} - \boldsymbol{w} \|_2^2 + \frac{1}{\mu + L} \|\nabla f(\boldsymbol{v}) - \nabla f(\boldsymbol{w}) \|_2^2,
	\end{equation*}
	thus concluding the proof.
\end{proof}
\clearpage	
\vskip 0.2in
\bibliography{sample}

\end{document}